%% file: main_advpaper.tex
\documentclass[11pt]{article}
\usepackage{url}
\usepackage{fullpage}
\usepackage{color}
\usepackage{amsmath,amssymb,amsthm}
\usepackage{amsmath,amssymb,mathtools}
\usepackage{bm}
\usepackage{amsmath}
\usepackage{grffile}
\usepackage{amssymb}
\usepackage{latexsym}
\usepackage{verbatim}
\usepackage{subfigure}
\usepackage{graphicx}
\usepackage{psfrag}
\usepackage{xcolor}
\usepackage{lipsum}
\usepackage{authblk}
\usepackage{blindtext}
\usepackage{mathrsfs} 
\usepackage{geometry}
\renewcommand{\cal}{\mathcal}

\newcommand\cB{{\mathcal B}}
\newcommand{\cC}{{\cal C}}
\newcommand{\cD}{{\cal D}}

\newcommand{\cL}{{\cal L}}

\newcommand{\cN}{{\cal N}}
\newcommand{\cT}{{\cal T}}

\newcommand{\cP}{{\cal P}}

\newcommand{\cR}{{\mathcal R}}

\newcommand{\cU}{{\mathcal U}}

\newcommand{\bma}{{\bm{a}}}

\newcommand{\bmu}{{\bm{u}}}
\newcommand{\bmv}{{\bm{v}}}
\newcommand{\bmw}{{\bm{w}}}
\newcommand{\bmx}{{\bm{x}}}

\newcommand{\bmz} {{\bm {z}}}

\newcommand{\bms}{\bm s}

\newcommand{\bmW}{{\bm W}}

\newcommand{\bmdelta}{{\bm \delta}}

\newcommand{\bE}{\mathbb{E}}

\newcommand{\bN}{\mathbb{N}}
\newcommand{\bP}{\mathbb{P}}

\newcommand{\bR}{{\mathbb R}}
\newcommand{\bS}{\mathbb S}

\newcommand{\bB}{{\mathbb B}}

\renewcommand{\leq}{\leqslant}
\renewcommand{\geq}{\geqslant}

\newtheorem{theorem}{Theorem}[section]

\newtheorem{lemma}{Lemma}[section]
\newtheorem{corollary}{Corollary}[section]
\newtheorem{remark}{Remark}

\usepackage{algorithm, algorithmic}
\usepackage{cite}

\DeclareMathOperator{\argmin}{argmin}

\def\bxi{\bm{\xi}}
\def\bvarsigma{\bm{\varsigma}}

\begin{document}
\thispagestyle{empty}
\title{Towards Understanding the Dynamics of the First-Order Adversaries}

\author{Zhun Deng\thanks{Harvard University, zhundeng@g.harvard.edu}
\qquad Hangfeng He\thanks{University of Pennsylvania, hangfeng@seas.upenn.edu}
\qquad Jiaoyang Huang\thanks{Institute of Advanced Study, jiaoyang@ias.edu}
\qquad Weijie J.~Su\thanks{University of Pennsylvania, suw@wharton.upenn.edu}}

\date{}
\maketitle


\input{abs_advpaper.tex}

\input{intro_advpaper.tex}

\input{method_advpaper.tex}


\bibliography{adv_cite}
\bibliographystyle{plain}

\input{app_advpaper.tex}

\end{document}

%% file: abs_advpaper.tex
\abstract

An acknowledged weakness of neural networks is their vulnerability to adversarial perturbations to the inputs. To improve the robustness of these models, one of the most popular defense mechanisms is to alternatively maximize the loss over the constrained perturbations (or called adversaries) on the inputs using projected gradient ascent and minimize over weights. In this paper, we analyze the dynamics of the maximization step towards understanding the experimentally observed effectiveness of this defense mechanism. Specifically, we investigate the non-concave landscape of the adversaries for a two-layer neural network with a quadratic loss. Our main result proves that projected gradient ascent finds a local maximum of this non-concave problem in a polynomial number of iterations with high probability.  To our knowledge, this is the first work that provides a convergence analysis of the first-order adversaries. Moreover, our analysis demonstrates that, in the initial phase of adversarial training, the scale of the inputs matters in the sense that a smaller input scale leads to faster convergence of adversarial training and a ``more regular'' landscape. Finally, we show that these theoretical findings are in excellent agreement with a series of experiments.

%% file: intro_advpaper.tex
\section{Introduction}
Neural networks have achieved remarkable success in many fields such as image recognition \cite{he2016deep} and natural language processing \cite{devlin2018bert}. However, it has been recognized that neural networks are not robust against adversarial examples -- prediction labels can be easily manipulated by human imperceptible perturbations \cite{goodfellow2014explaining, szegedy2013intriguing}. In response, many defense mechanisms have been proposed against adversarial attacks such as input de-noising \cite{guo2017countering}, randomized smoothing \cite{ilyas2019adversarial}, gradient regularization \cite{papernot2017practical}, and adversarial training \cite{madry2017towards}. Among these, one of the most popular techniques is adversarial training, which proposes to add adversarial examples into the training set as a way of improving the robustness.

As oppose to earlier work that only adds adversarial examples several times during the training phase, more recently, \cite{madry2017towards} propose to formulate adversarial training through the lens of robust optimization, showing substantial improvement. More precisely, robust optimization for a loss function $L$ in its simples setting takes the form
\begin{equation*}\label{eq:ERM}
\min_{\bm\theta\in\Theta}\bE_{(\bmx,y)\sim\cD}\max_{\|\bmdelta\|_p\leq \varepsilon}L(\bm\theta,\bmx+\bmdelta,y),
\end{equation*}
where $\bm\theta\in\Theta$ is the parameter and $(\bmx,y)$ are the input and label following some unknown joint distribution $\cD$. The inner maximization problem is to find an adversarial example, where $\bmdelta$ is an adversarial perturbation with $l_p$ norm constraint for some integer $1\leq p\leq\infty$. 

For neural networks, the inner maximization problem is typically non-concave and the most commonly used method in implementation is through the first-order method --  projected gradient ascent. However, as pointed out by \cite{wang2019convergence}, the degree to which it solves the inner maximization problem has not been thoroughly understood. While there are several papers providing great theoretical insights to the convergence of adversarial training, they either formulate the inner maximization problem as maximizing the first order taylor expansion of the loss \cite{wang2019convergence}, or treat the inner maximization problem abstractly as a general function of data and study the convergence in the neural tangent kernel regime \cite{gao2019convergence}. In our paper, we make the first step to analyze the dynamics of projected gradient ascent of neural networks.

The first question is about the effectiveness of projected gradient ascent. To prove the effectiveness, we need to consider the time cost of using projected gradient ascent in the inner maximization problem. In  \cite{madry2017towards}, one claim is that using projected gradient ascent can find the adversaries rapidly. That claim is very important since adversarial training usually takes much longer than usual training due to the inner maximization problem. Specifically, if we use gradient method in the alternative optimization problem for both inner maximization and outer minimization, and denote the number of epochs taken to find adversaries with given weights  by $n_1$, number of updates of weights by $n_2$, then the epochs taken by the adversarial training is $n_1n_2$. To make the time cost of adversarial training bearable, the fact that $n_1$ is not large plays a key role here.

Another issue about effectiveness is whether the projected gradient ascent can truly find a local maximum and not be stuck at a saddle point. In \cite{madry2017towards}, Madry et al. claims the loss (as a function of model parameters) typically has many local maximums with very similar values. So, if the projected gradient ascent truly finds a local maximum, the effectiveness of the adversarial training is trustworthy.

We summarize our first question below.

\noindent\textbf{Questions 1.} \textit{Does projected gradient ascent truly find a local maximum rapidly? }

The second question we try to explore is whether the scale of inputs matters. In the adversarial training, $\varepsilon$'s scale is usually in proportional to the scale of input $\bmx$:
\begin{equation*}\label{eq:intervene}
\varepsilon = r\bE\|\bmx\|_p.
\end{equation*}
For adversarial attacks on images, the ratio $r$ is supposed to be small, so as to reflect the fact that the attacks are visually imperceptible.

For fixed $r>0$,  $\varepsilon$ and the input scale $\bE\|\bmx\|_p$ are closely related -- a smaller input scale implies a smaller $\varepsilon$.  In the implementation of image recognition using neural networks, people usually rescale the image pixels to $[0,1]$ or $[-1,1]$.  While that seems not affecting regular optimization, it may affect adversarial training. So, we have the following question.

\noindent\textbf{Questions 2.} \textit{When we fix the ratio $r$, do smaller input scales (implying smaller $\varepsilon$) help optimization of adversarial training?}

If the answer to Question 2 is positive, it will be helpful in the future applications to rescale the inputs to a smaller scale. 

Both questions above have not been studied yet due to the highly non-concave landscape of adversaries.

\subsection{Our contributions}
Our analysis provides answers to Question 1 and 2 for the initial phase of the adversarial training, i.e. the weights are drawn from Xavier initialization \cite{glorot2010understanding}. Even for this simple case, nothing has been discussed theoretically so far. 

In Section \ref{sec:convergenceanalysis} and \ref{sec:proofsketch}, we provide the answer to Question 1 by showing projected gradient ascent indeed can find a local maximum rapidly by providing a convergence theorem.
\begin{theorem}[\textbf{Informal}]
Projected gradient ascent can obtain an approximate local maximum, which is close to a true local maximum on the sphere in polynomial number of iterations when the learning rate is small enough. If we further allow learning rate shrinking with time, projected gradient ascent can converge to a local maximum. 
\end{theorem}
In Section \ref{sec:implications and extensions}, we answer Question 2 by showing a smaller scale helps in the perspectives of landscapes and convergence of trajectories. From our theory, we show a smaller input scale helps the trajectory converge faster if we had a bad initialization. Besides, a smaller input scale makes the local maximums concentrate better, which can partially explain why the loss value of local maximums share similar values \cite{madry2017towards}.  Lastly, we verify the previous claims by extensive numerical experiments. 

Our work mainly focuses on the initial phase of adversarial learning, which may be a good start towards understanding the first-order adversaries.

\subsection{Related work}
\noindent\textbf{Adversarial attack and defense}\quad Besides projected gradient ascent, some other have also been proposed to generate adversarial examples, such as FGSM \cite{goodfellow2014explaining}, l-BFGS \cite{szegedy2013intriguing} and C\& W attack \cite{carlini2017towards}. Also, some attacks are proposed to attack black-box models, in order to defend against those attacks, many defense mechanisms have been proposed \cite{chen2017zoo, brendel2017decision}.  However, many of these defense models have been evaded by new attacks \cite{athalye2018obfuscated} except \cite{madry2017towards}. Besides, a line of work focus on providing certified robustness and robustness verification \cite{wong2018scaling, weng2018towards, zhang2018efficient}, which also provide useful insights theoretically.

\noindent\textbf{Adversarial training}\quad The first work to propose adversarial training is \cite{goodfellow2014explaining}, in which the authors advocate adding adversarial examples during training to improve the robustness. In \cite{madry2017towards}, the authors use projected gradient ascent to find adversaries and reach a state of art performance. However, as we mentioned before, running projected gradient method is very slow, and some work \cite{shafahi2019adversarial} intend to solve this problem. Besides, the introducing of adversarial training also motivates a line of theoretical work, such as \cite {agarwal2018learning, liu2019rob, yin2018rademacher}. However, none of them address the inner maximization problem using projected gradient ascent. 

\noindent\textbf{Non-convex optimization}\quad Non-convex optimization is notoriously hard to analyze. However, some work provide valuable guide. In \cite{ge2015escaping}, the authors analyze the dynamics of noisy gradient descent in the non-convex setting. Some following work including \cite{du2017gradient} show gradient descent can take very long to escape saddle point but noisy gradient descent does not, and \cite{jin2017escape} shows noisy gradient descent can converge to a second order stationary point vert fast. In our setting, we do not need extra noise, but can still yield a good convergence result.

%% file: method_advpaper.tex
\section{Preliminaries}
\paragraph{Notations}
Throughout the paper, we use $[n]$ to denote $\{1,2,\cdots,n\}$ and use $\|\cdot\|_p$ to denote $l_p$ norm. In particular, for $l_2$ norm, we use $\|\cdot\|$ or $\|\cdot\|_2$ exchangeably. For any function $L:\bR^d\mapsto \bR$, $\nabla L$ and $\nabla^2 L$ denote the gradient vector and Hessian matrix respectively. We use $\bB$ to denote ball and $\bS$ to denote sphere. We denote $\angle [\bmu,\bmv]=\bmu^T\bmv/(\|\bmu\|\|\bmv\|)$, which is the cosine value of the angle between the two vectors $\bmv$ and $\bmu$. For a function $h(\bm x)$, we sometimes use shorthand $\partial h(\bm x)$ for gradient $\partial h(\bm x)/\partial \bm x$.

\paragraph{Setup}
\label{subsec:setup}
Recall adversarial learning aims to solve the robust optimization of loss function $L$:
\begin{equation*}
\min_{\bm\theta\in\Theta}\bE_{(\bmx,y)\sim\cD}\max_{\|\bmdelta\|_p\leq \varepsilon}L(\bm\theta,\bmx+\bmdelta,y),
\end{equation*}
where $\bm\theta\in\Theta$ is the parameter, $(\bmx, y)\in \bR^d\times \bR$ is $d$-dimensional input and scalar output, which follows a joint distribution $\cD$. The corresponding empirical version for samples $\{\bmx_i,y_i\}_{i=1}^n$ is 
\begin{equation}\label{eq:sep}
\min_{\bm\theta\in\Theta}\frac{1}{n}\sum_{i=1}^n\max_{\forall i\in[n],\|\bmdelta_i\|_p\leq \varepsilon}L(\bm\theta,\bmx_i+\bmdelta_i,y_i).
\end{equation}
For fixed $\bm\theta$, solving the optimization problem (\ref{eq:sep}) can be optimized as $n$ different optimization problems separately: for each $\bm x_i$, we need to obtain a corresponding $\bm\delta_i$. In this paper, we focus on studying the convergence rate of finding adversaries, i.e. maximizing $\bmdelta\in\bR^d$ when the constraint is the $l_2$-norm and the loss is the quadratic loss of shallow neural network:
\begin{equation}\label{eq:constraintmax}
\max_{\bmdelta}L(\bm\theta,\bmx+\bmdelta,y)=(y-f(\bma,\bmW,\bmdelta+\bmx))^2,
~\text{s.t.}~
\|\bmdelta\|_2^2\leq\varepsilon^2.
\end{equation}
Here, $f$ is a two-layer neural network:
$$f(\bma,\bmW,\bmdelta+\bm\bmx)=\sum_{r=1}^m a_r\sigma(\bmw_r^T(\bmx+\bmdelta)).$$
In the above equation, $\bma=(a_1,a_2,\cdots,a_m)^T$ is an $m$-dimensional vector, $\bmW=(\bmw_1,\cdots,\bmw_m)$ is an $m\times d $-matrix and $\bm\theta=(\bma^T,\text{Vec}(\bmW)^T)^T$, where Vec$(\cdot)$ is the vectorization operator. We use $\sigma$ to denote the softplus activation function such that $\sigma(x)=\log(1+e^x).$

We study the projected gradient ascent:
\begin{equation*}\label{eq:pga}
\bmdelta_{t+1}=\cP_{\bB(\bm0,\varepsilon)}\Big[\bmdelta_{t}+\eta\frac{\partial L(\bmdelta_{t})}{\partial \bmdelta_{t}}\Big],\quad t\geq0,
\end{equation*}
where $\bB(\bm0,\varepsilon)$ is a ball centered at $\bm0$ with radius $\varepsilon$ in Euclidean distance, and $\cP$ is the projection operator. $\bmdelta_0$ is uniformly sampled in the ball $\bB(\bm0,\varepsilon)$.

In this paper, we always consider the problem under the following settings unless we state explicitly otherwise. \begin{itemize}
\item[1.] $\bmw_r$'s are i.i.d drawn from $d$-dimensional Gaussian $\cN(0,\kappa^2I)$, where $0<\kappa\leq 1$ controls the magnitude of initialization.
\item[2.]  $a_r$'s are i.i.d drawn from Bernoulli distribution, which take $\pm \gamma$ with $1/2$ probability.
\item[3.] There exist $\cL, \cU>0$ such that $\cL<|y-f(\bma,\bmW,\bmdelta+\bm\bmx)|<\cU$ for all $\bmdelta\in\bB(\bm0,\varepsilon)$.
\item[4.] $\bmdelta_0$ is initialized by drawing from a uniform distribution over  $\bB^\circ(\bm0,\varepsilon)$, where $\bB^\circ$ stands for the interior of the ball $\bB$.
\end{itemize}
In this paper, we will take the parameters according to Xavier initialization, which means $\kappa=d^{-1/2}$ and $\gamma=m^{-1/2}$.
\begin{remark}
We study the case when the weights are drawn from commonly used distributions for initialization. Our analysis can be viewed as studying the dynamics of finding adversaries in the initial phase of training.
\end{remark}

\section{Main Results}\label{sec:convergenceanalysis}
We present our main results on the convergence of projected gradient descent (PGD) in this section. Since the objective of optimization is $\bmdelta$, we use $L(\bmdelta)$ for loss and we denote the constraint as $c(\bmdelta)=\|\bmdelta\|^2-\varepsilon^2$. For convenience, we consider the minimization version:
$$L(\bmdelta)=-(y-f(\bma,\bmW,\bmdelta+\bm\bmx))^2.$$
The original problem in (\ref{eq:constraintmax}) is equivalent to:
$$\min_{\bmdelta}L(\bmdelta),\quad \text{s.t.}~c(\bmdelta)\leq 0.$$
Then, the iterative optimization algorithm used becomes the projected gradient descent (PGD) 
\begin{equation*}\label{pgd}
\bmdelta_{t+1}=\cP_{\bB(\bm0,\varepsilon)}\Big[\bmdelta_{t}-\eta\frac{\partial L(\bmdelta_{t})}{\partial \bmdelta_{t}}\Big],\quad t\geq 0.
\end{equation*}
Here, we provide the formal statement of our main results.
\begin{theorem}[\textbf{Main Theorem}]\label{thm:main}
Suppose $m=\Omega(d^{5/2})$, there exists $\varepsilon_{\max}(m)=\Theta \big((\log m)^{-2}\big)$ and $\eta_{\max}(m,\varepsilon)=\min\{\Theta \big((\log m)^{-2}\big),\varepsilon^2\}$, if $\varepsilon<\varepsilon_{\max}(m)$, for any $\eta<\eta_{\max}(m,\varepsilon)$, with high probability, in $O(\eta^{-2})$ iterations, projected gradient descent will output a point $\bmdelta_t$ on the sphere which is $O(\eta^{1/2})$ close to some local minimum $\bmdelta^*$. \end{theorem} 
\begin{remark}
Our width requirement is much smaller compared to the results with respect to neural tangent kernels \cite{du2018gradient, jacot2018neural}. The latter one requires $m=O(\text{poly}(n))$, where $n$ is the samples size. Notice the scale of $\varepsilon$ only requires to be upper bounded by $O(\text{poly}((\log m)^{-1})$, under that requirement, the activation function will be activated along the update of $\bmdelta$ with constant probability when $\|\bmx\|$ is small. 
\end{remark}
\begin{corollary}[\textbf{Shrinking learning rate}]\label{col:main}
Under the assumptions of Theorem \ref{thm:main}, for $\tilde{t}$ satisfying $\bmdelta_{\tilde{t}}\in \varepsilon\bS^{d-1}$ and the tangent component of $\partial f(\bmdelta_{\tilde{t}})$ (for every point on the sphere, the tangent component of a vector is its projection to the tangent plane at that point) being smaller than $\eta^{1/2}$, let $D_s:=\|\bmdelta_{\tilde{t}+s}-\bmdelta^*\|^2$, if we shrink the learning rate after $\tilde{t}$ , in a way that
\begin{equation*}
\bmdelta_{\tilde{t}+s+1}=\cP_{\bB(\bm0,\varepsilon)}\Big[\bmdelta_{\tilde{t}+s}-\eta_s \partial L(\bmdelta_{\tilde{t}+s})\Big],\quad t\geq0,~s\geq0,
\end{equation*}
for $\eta_0<\eta$, as long as  $\eta_s\rightarrow 0$ as $s\rightarrow\infty$ and  $\Pi_{i=0}^{k}(1-\gamma\eta_i/2)\rightarrow 0$ as $k\rightarrow \infty$, we will have $D_s\rightarrow 0$. Furthermore, if 
\begin{equation}\label{eq:stepshrink}
\eta_s\Pi_{i=0}^k(1-\frac{\beta\eta_{s+i}}{2})\leq \eta_{s+k+1}
\end{equation}
for all $s,k\in\bN$, where $\beta$ is a constant depending on $(d, m,\varepsilon, \eta)$ and can be calculated explicitly,  then for all $s\in \bN$, 
$$D_s\leq O(\eta_s).$$
\end{corollary}

\begin{remark}
One concrete example satisfying Eq. (\ref{eq:stepshrink}) is the following one: if $\eta_s=2/(\beta s+\beta z)$ for  large enough integer $z$, 
$$D_s\leq O(\frac{1}{z+s}).$$

\end{remark}

\subsection{Our interpretation}\label{subsec:interpretation}
Our results state that for a wide enough one hidden layer neural network, if the attack size $\varepsilon$ is small, then we can choose small enough learning rate, such that the trajectory of PGD can quickly reach a point that is very close to one of the minimizers. Besides, the minimizer is located on the sphere with high probability. The theory can partially explain the observation in \cite{madry2017towards}: it does not take too many iterations to find an adversary, which is the key to guarantee the time cost of robust optimization modest. Also, our theory is consistent with the observation that the PGD will end up on the sphere for most samples in the implementation of adversarial training.

\section{Proof Sketch}\label{sec:proofsketch}
In this section, we briefly sketch our proof. We show with high probability, the gradient is non-vanishing in the ball. Meanwhile, on the sphere, there is no saddle points. Besides, the trajectory will not get stuck near local maximums and can converge to a local minimum in polynomial number of iterations. 
\begin{lemma}[\textbf{Dynamics in the ball}]\label{lemma:inball}
For $m=\Omega(d^{5/2})$, there exists $\varepsilon_{\max}(m)=\Theta \big((\log m)^{-1/2}\big)$ and $\eta_{\max}(m,\varepsilon)=\min\{\Theta \big((\log m)^{-1}\big),\varepsilon^2\}$, if $\varepsilon<\varepsilon_{\max}(m)$, $\eta<\eta_{\max}(m,\varepsilon)$
, with high probability, whenever $\bmdelta_{t+1}\in\bB^\circ(\bm0,\varepsilon)$
\begin{equation*}
L(\bmdelta_{t+1})-L(\bmdelta_t)\leq -\Omega(\eta).
\end{equation*}
\end{lemma}
The above lemma shows the trajectory is very unlikely to terminate in the ball since the $(t+1)$-th step can make progress if $\bmdelta_{t+1}\in\bB^\circ(\bm0,\varepsilon)$. 

Next, we focus on studying the dynamics on the sphere. For constrained optimization, we can locally transform it into an unconstrained problem by introducing Lagrangian multipliers:
$$L(\bmdelta,\lambda)=L(\bmdelta)-\lambda c(\bmdelta).$$
Under some regularity conditions, we can obtain the Lagrangian multiplier $\lambda^*(\cdot)$:
$$\lambda^*(\bmdelta)= \argmin_{\lambda}\|\partial L(\bmdelta)-\lambda\partial c(\bmdelta)\|.$$
There are two key quantities. The first quantity can be viewed as an approximate gradient when we have constraints, which we will denote as $\Gamma$:
$$\Gamma(\bmdelta)=\partial L(\bmdelta,\lambda)|_{(\bmdelta,\lambda^*(\bmdelta))}=\frac{\partial L(\bmdelta)}{\partial \bmdelta}-\lambda^*(\bmdelta)\frac{\partial c(\bmdelta)}{\partial \bmdelta}.$$
Another important quantity can be viewed as the approximate Hessian of constraint optimization: 
$$\Xi(\bmdelta)=\partial^2 L(\bmdelta,\lambda)|_{(\bmdelta,\lambda^*(\bmdelta))}=\frac{\partial^2 L(\bmdelta)}{\partial \bmdelta^2}-\lambda^*(\bmdelta)\frac{\partial^2 c(\bmdelta)}{\partial \bmdelta^2}.$$
For $\bmdelta,\bmdelta'\in\varepsilon \bS^{d-1}$, if $\partial^2L(\bmdelta,\lambda^*)$ is $\rho$-Lipschitz, i.e. $\|\partial^2L(\bmdelta_a,\lambda^*)-\partial^2L(\bmdelta_b,\lambda^*)\|\leq\rho\|\bmdelta_a-\bmdelta_b\|$ for all $\bmdelta_a,\bmdelta_b\in\bB(\bm0,\varepsilon)$, we can obtain
\begin{align*}
L(\bmdelta,\lambda^*)&\leq L(\bmdelta',\lambda^*)+\partial L(\bmdelta',\lambda^*)^T(\bmdelta-\bmdelta')+\frac{1}{2}(\bmdelta-\bmdelta')^T\partial^2L(\bmdelta',\lambda^*)(\bmdelta-\bmdelta')+\frac{\rho}{6}\|\bmdelta-\bmdelta'\|^3.
\end{align*}
Since $\bmdelta,\bmdelta'$ are on the sphere, we know $L(\bmdelta,\lambda^*)=L(\bmdelta)$ and $L(\bmdelta',\lambda^*)=L(\bmdelta')$, we have 
\begin{align}
\begin{split}
L(\bmdelta)\leq &L(\bmdelta')+\Gamma(\bmdelta')^T(\bmdelta-\bmdelta')+\frac{1}{2}(\bmdelta-\bmdelta')^T\Xi(\bmdelta')(\bmdelta-\bmdelta')+\frac{\rho}{6}\|\bmdelta-\bmdelta'\|^3. \label{eq:taylor}
\end{split}
\end{align}
Further, we denote $\cT(\bmdelta)$ as the tangent space at $\bmdelta$ on the sphere, and $\cP_{\cT(\bmdelta)}$ is the operator for projection to the tangent space  $\cT(\bmdelta)$. The projected gradient descent can be approximated in the manner stated in the following lemma.
\begin{lemma}[\textbf{Approximation of PGD}]\label{lemma:geometry}For any $\hat{v}\in \bS^{d-1}$, let $\tilde{\bmdelta}_1=\bmdelta_0+\eta\hat{v}$ and $\tilde{\bmdelta}_2=\bmdelta_0+\eta\cP_{\cT_0}\cdot\hat{v}$
\begin{equation*}
\|\cP_{\bB(\bm0,\varepsilon)}(\tilde{\bmdelta}_1)-\tilde{\bmdelta}_2\|\leq\frac{4\eta^2}{\varepsilon}.
\end{equation*}
\end{lemma}
It is worth noting that $\Gamma(\bmdelta)$ is actually the tangent component of $\partial{f(\bmdelta)}$
\begin{equation*}
\Gamma(\bmdelta)=\cP_{\cT(\bmdelta)}\cdot\partial f(\bmdelta).
\end{equation*}
As a result, for $\bmdelta_t\in \varepsilon \bS^{d-1}$ 
\begin{equation}\label{eq:geometry}
\|\bmdelta_{t+1}-(\bmdelta_t-\eta\Gamma(\bmdelta_t))\|\leq\frac{4\eta^2}{\varepsilon}.
\end{equation}
We can use the above Eq. (\ref{eq:taylor}) and (\ref{eq:geometry}) to calculate the progress at each step. Thus, in order to analyze the progress, we only need to carefully analyze $\Gamma$ and $\Xi$. In the following paragraph, we discuss $\Gamma$ and $\Xi$ case by case.

For each point on the sphere, we loosely define ``near'' and ``away from'' local optimums by looking into the angle between the gradient and the spherical normal vector. If the gradient is parallel to the spherical normal vector at a point on the sphere, then the point is a fixed point for projected gradient descent. It is either a local optimum or a saddle point. We will show such points are not saddle points under some regularity conditions. Since $c(\bmdelta)=\|\bmdelta\|^2-\varepsilon^2$, the unit spherical normal vector is $\bmdelta/\|\bmdelta\|$ at each point on the sphere and the cosine value of the angle we are looking at is $\angle[\partial f(\bmdelta),\bmdelta]$. If $\angle[\partial f(\bmdelta),\bmdelta]$ is close to $\pm 1$, then such $\bmdelta$ is close to a critical point.
\begin{lemma}[\textbf{Away from critical points on the sphere}]\label{lemma:awayopt}
For $m=\Omega(d^{5/2})$, there exists a threshold $\varepsilon_{\max}(m)=\Theta((\log m)^{-1})$, if $\varepsilon<\varepsilon_{\max}$, with high probability, for any  $\bmdelta\in\varepsilon\bS^{d-1}$ and any $0\leq\beta\leq 1$ such that
$$\angle[\partial f(\bmdelta),\bmdelta]\leq \beta,$$
we have 
$$\|\Gamma(\bmdelta)\|\geq\sqrt{1-\beta^2}\|\partial f(\bmdelta)\|\geq \cL B_l \sqrt{1-\beta^2},$$
where $B_l$ is of order $\Theta (1)$.
\end{lemma}
Recall $\cL$ is the lower bound such that $|y-f(\bma,\bmW,\bmdelta+\bm\bmx)|>\cL$ for all $\bmdelta\in\bB(\bm 0,\varepsilon)$. The above lemma shows if the trajectory is away from critical points, each step can decrease the loss value by $-\Omega(\eta)$ since $\bmdelta_{t+1}\approx\bmdelta_t-\eta\Gamma(\bmdelta_t)$ and $L(\bmdelta_{t+1})\leq L(\bmdelta_t)+\Gamma(\bmdelta_t)^T(\bmdelta_{t+1}-\bmdelta_t)+O(\|\bmdelta_{t+1}-\bmdelta_t\|^2)$.

The hard case is when the trajectory is near a critical point on the sphere. We will first show that the critical points on the sphere are not saddle points under some regularity conditions. 
\begin{lemma}[\textbf{Near critical points on the sphere}]\label{lemma:nearopt}
For $m=\Omega(d^{5/2})$, there exists a threshold $\varepsilon_{\max}(m)=\Theta((\log m)^{-1})$, if $\varepsilon<\varepsilon_{\max}$, with high probability, there exists universal constants $\phi,\gamma>0$, for any  $\bmdelta\in\varepsilon\bS^{d-1}$, such that
$$\angle[\partial f(\bmdelta),\bmdelta]\geq \phi,$$
then for all $\|\bmv\|=1$,
$$\text{sgn}\big((y-u)\bmdelta^T\partial f(
\bmdelta)\big)\cdot\bmv^T\Xi\bmv\geq \gamma.$$
\end{lemma}
Lemma \ref{lemma:nearopt} implies $\Xi$ is either positive definite or negative definite near a critical point, thus, none of the critical points on the sphere are saddle points.

Since near a local minimum, the trajectory can converge to that local minimum by traditional analysis technique, the only thing left to deal with is when the trajectory is near local maximums. The following lemma states the trajectory will not be stuck near any local maximum with high probability.

We denote the set $\Delta^{-}_{\eta}=\{\bmdelta:\angle[\partial f(\bmdelta),\bmdelta]\leq -1+\sqrt{\eta}/(\cL B_l),\bmdelta\in\varepsilon\bS^{d-1}\}$ and  $\Delta^{+}_{\eta}=\{\bmdelta:\angle[\partial f(\bmdelta),\bmdelta]\geq 1-\sqrt{\eta}/(\cL B_l),\bmdelta\in\varepsilon\bS^{d-1}\}$. Notice that $\angle[\partial f(\bmdelta),\bmdelta]=\pm 1$ when the spherical normal vector is parallel to the gradient at $\bmdelta$. Thus, for small $\eta$, the two sets are the collections of points near local maximums and local minimums respectively.
\begin{lemma}[\textbf{Trajectory and local optimums}]\label{lemma:trajectory}
For learning rate $\eta$ such that $\eta<\min\{1,\cL B_l\}$, if
\begin{align}\label{eq:trajectory}
\arccos\big(\angle[\partial f(\bmdelta),\partial f(\bmdelta')]\big)+\arccos\left(\sqrt{\frac{(\cL B_l)^2-\eta}{(\cL B_l)^2}}\right)\leq \frac{\pi}{4}
\end{align}
for all $\bmdelta,\bmdelta'\in\varepsilon\bS^{d-1}$, the trajectory initialized by drawing from a uniform distribution over  $\bB^\circ(\bm0,\varepsilon)$ will never reach $\Delta^{-}_{\eta}$. Meanwhile, if there exists $t^*$ such that $\bmdelta_{t^*}\in \Delta^{+}_{\eta}$, then for all $t\geq t^*$, $\bmdelta_{t}\in \Delta^{+}_{\eta}$.
\end{lemma}
From the discussions above, it is easy to see Lemma \ref{lemma:trajectory} holds for small enough $\varepsilon$ and $\eta$. The above lemma states the trajectory will not be stuck near local maximums and $\|\Gamma(\bmdelta)\|\geq \sqrt{\eta}$ if $\bmdelta\notin \Delta^{+}_{\eta}$ for $\bmdelta\in\varepsilon\bS^{d-1}$. That can ensure $L(\bmdelta_{t+1})-L(\bmdelta_{t})\leq-\Omega(\eta^2)$ for $\bmdelta_{t},\bmdelta_{t+1}\in \varepsilon\bS^{d-1}$. As a result, the trajectory can constantly make progress until the trajectory reaches $\Delta^{+}_{\eta}$. Then, traditional techniques for convex optimization can be applied and gives us the final convergence result.

\begin{figure*}[t]
		\centering
		
		\subfigure[Input scale=0.01, ratio=0.1]{
			\centering
			\includegraphics[scale=0.21]{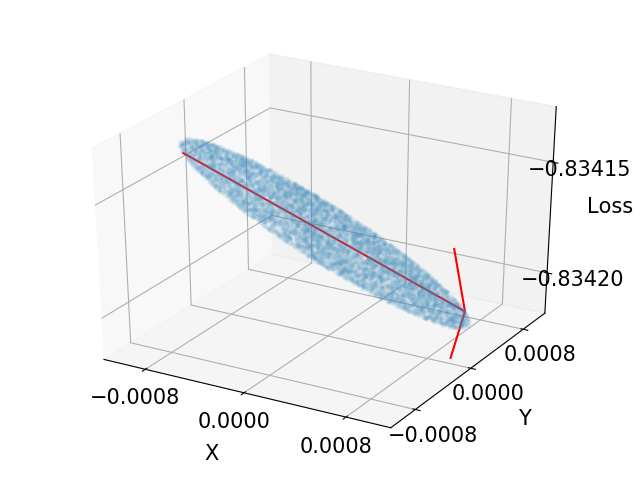}
			\label{fig:t-0.01-0.1}}
        \hspace{0.01in}
        \subfigure[Input scale=0.01, ratio=1.0]{
		\centering
		\includegraphics[scale=0.21]{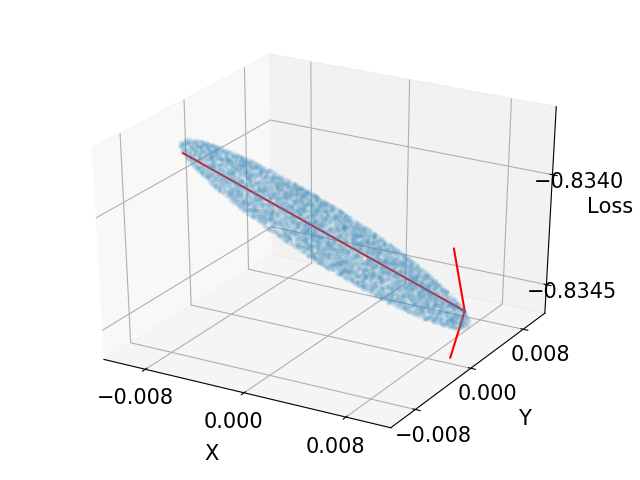}
		\label{fig:t-0.01-1.0}}
        \hspace{0.01in}
        \subfigure[Input scale=0.01, ratio=10]{
		\centering
		\includegraphics[scale=0.21]{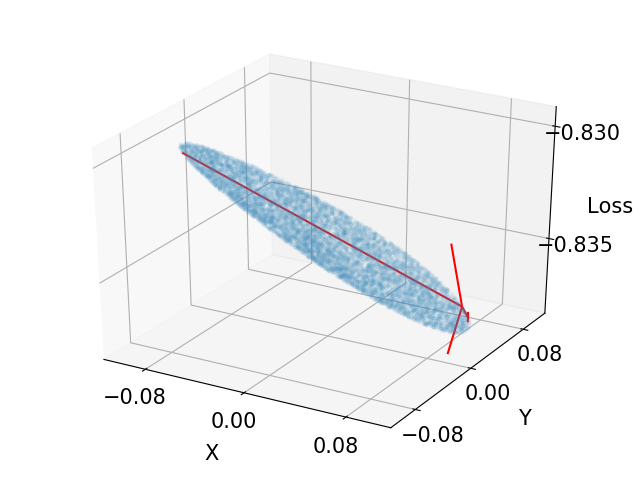}
		\label{fig:t-0.01-10}}
        \hspace{0.01in}
        
        \subfigure[Input scale=1.0, ratio=0.1]{
			\centering
			\includegraphics[scale=0.21]{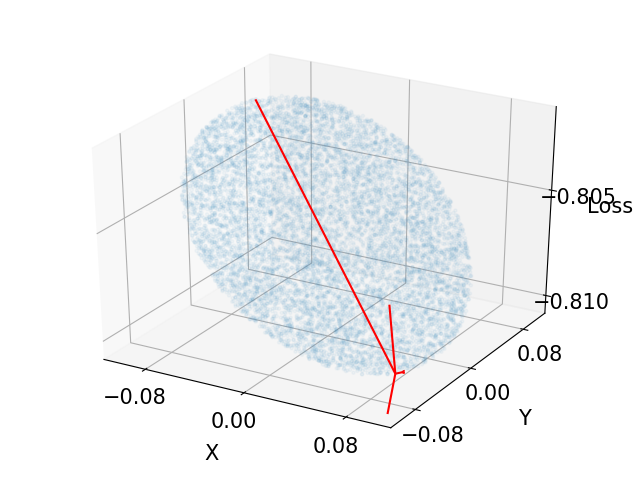}
			\label{fig:t-1.0-0.1}}
        \hspace{0.01in}
        \subfigure[Input scale=1.0, ratio=1.0]{
		\centering
		\includegraphics[scale=0.21]{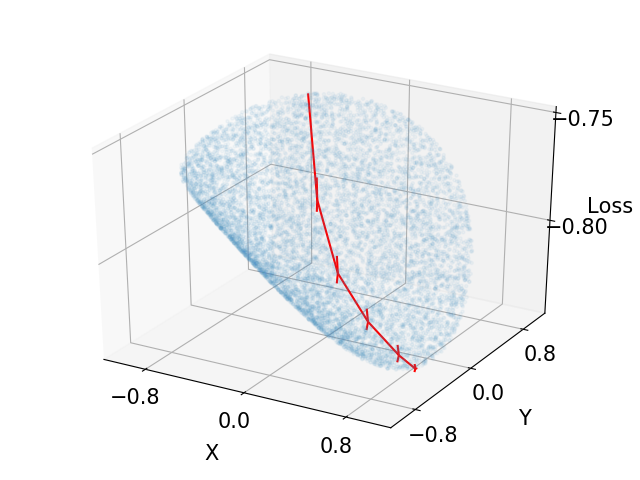}
		\label{fig:t-1.0-1.0}}
        \hspace{0.01in}
        \subfigure[Input scale=1.0, ratio=10]{
		\centering
		\includegraphics[scale=0.21]{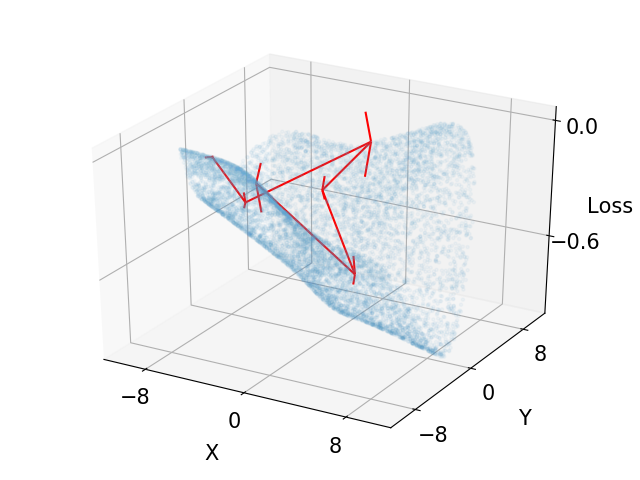}
		\label{fig:t-1.0-10}}
        \hspace{0.01in}
     
        \subfigure[Input scale=100, ratio=0.1]{
			\centering
			\includegraphics[scale=0.21]{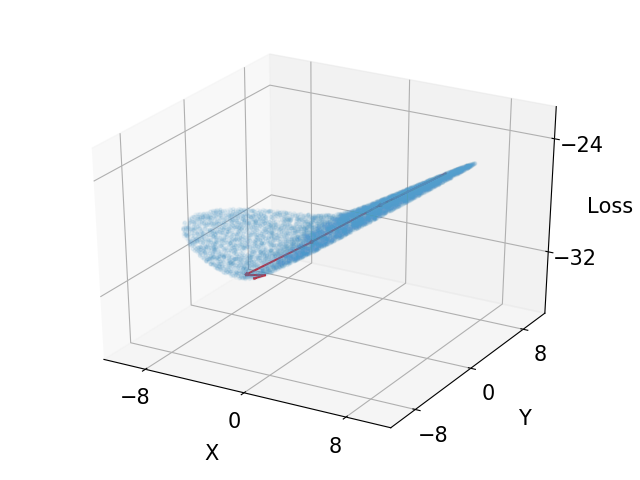}
			\label{fig:t-100-0.1}}
        \hspace{0.01in}
        \subfigure[Input scale=100, ratio=1.0]{
		\centering
		\includegraphics[scale=0.21]{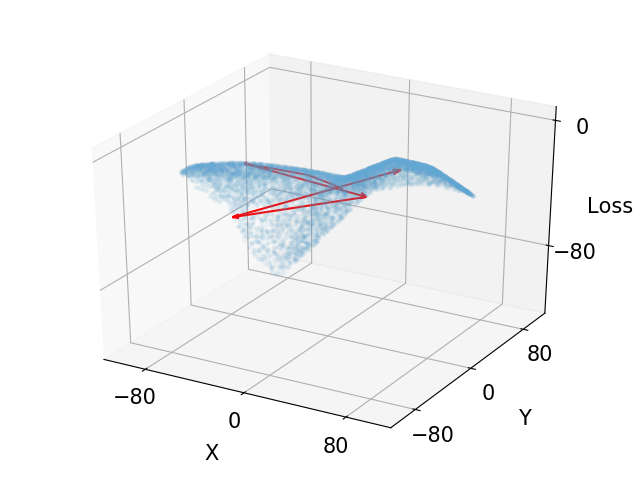}
		\label{fig:t-100-1.0}}
        \hspace{0.01in}
        \subfigure[Input scale=100, ratio=10]{
		\centering
		\includegraphics[scale=0.21]{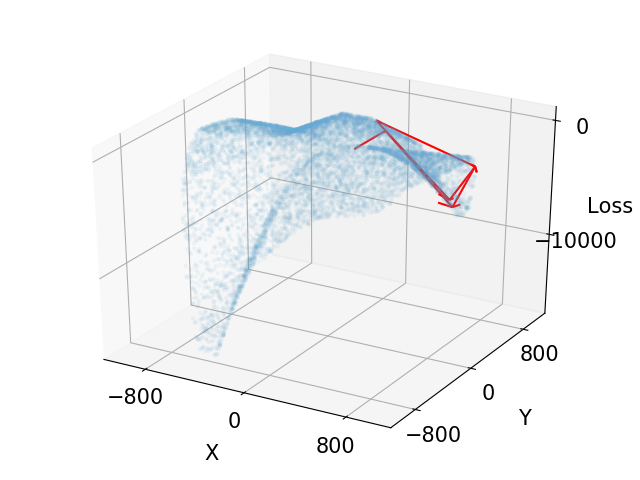}
		\label{fig:t-100-10}}
        \hspace{0.01in}
        
        \caption{Landscapes and trajectories on simulated data. We compare the landscapes and trajectories with three different input scales and three different perturbation ratios. If the input scale is small enough (i.e. $0.01$ here), the landscape has only one local minimum and PGD can easily escape the local maximal with few steps even with large perturbation ratio such as $10$. On the other hand, if the input scale is not small enough, we will have a less regular landscape with a lot of local minimums and it takes a lot of steps to escape from a local maximum. For large input scales, we have to reduce the perturbation ratio, so as to make the landscape become more regular and make escaping from the local maximums faster. Our simulations are based on two-dimensional inputs and two-layer neural networks. More details can be found in the supplementary materials.}
		\label{fig:trajectories}
\end{figure*}

\section{Implications and Extensions}\label{sec:implications and extensions}
So far, we have derived the theory about finding adversaries in the initial phase of adversarial training. Through our theoretical analysis, we also identify several interesting phenomena concerning the scale of input $\bmx$. In this section, we briefly discuss the implications of our theory on experiments and show how to extend our arguments to general losses.

\subsection{Scale, Landscape and Convergence}\label{subsec: scale}

In this subsection, we state the high level conclusions and the details of the theoretical results are left in the supplementary materials.

As we stated in the introduction, $\varepsilon$'s scale is usually formulated in proportional to the scale of input $\bmx$. In the empirical optimization (\ref{eq:sep})
\begin{equation*}
\min_{\bm\theta\in\Theta}\frac{1}{n}\sum_{i=1}^n\max_{\forall i\in[n],\|\bmdelta_i\|_p\leq \varepsilon}L(\bm\theta,\bmx_i+\bmdelta_i,y_i),
\end{equation*}
$\varepsilon$ takes the form 
$$\varepsilon = r\sum_{i=1}^n\frac{\|\bmx_i\|_p}{n}$$
for small $r>0$, where $r$ stands for a small  constant ratio.

\begin{figure*}[t]
		\centering
		\subfigure[Input scale=0.1]{
			\centering
			\includegraphics[scale=0.3]{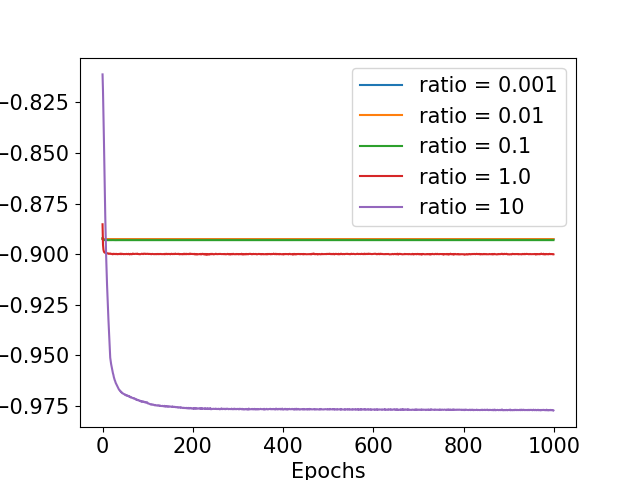}
			\label{fig:tr-0.1}}
        \hspace{0.01in}
        \subfigure[Input scale=10]{
		\centering
		\includegraphics[scale=0.3]{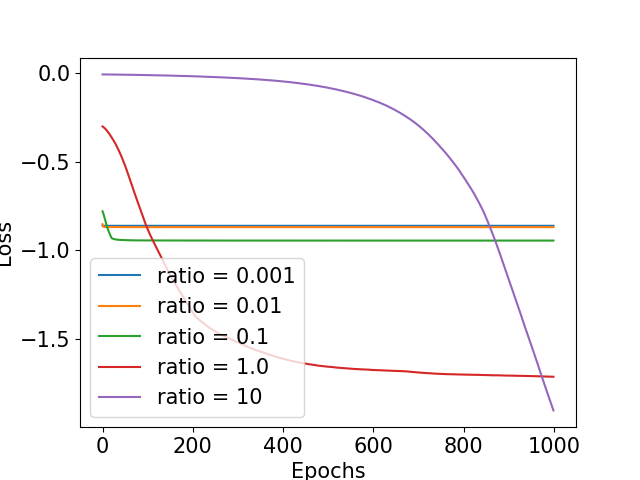}
		\label{fig:tr-10}}
        \hspace{0.01in}
        \subfigure[Input scale=1000]{
		\centering
		\includegraphics[scale=0.3]{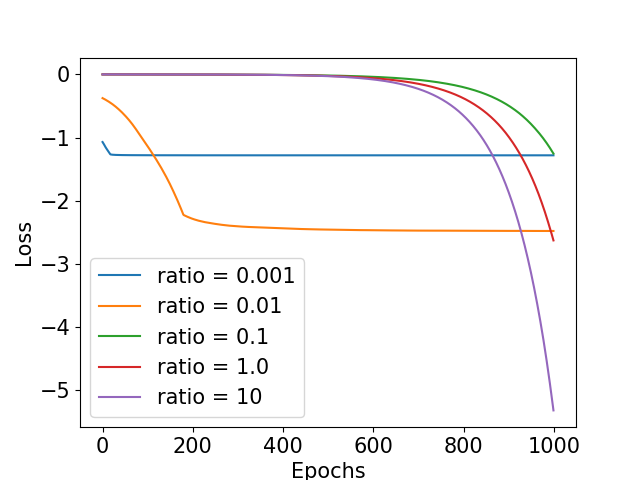}
		\label{fig:tr-1000}}
        \hspace{0.01in}
        \caption{Trajectories from local maxima to local minima on real-world data. We show the adversarial losses of each point on the trajectories from local maxima with three different input scales and five different perturbation ratios. For a fixed perturbation ratio, a smaller input scale means that escaping from local maxima is easier. If the input scale is small enough (i.e. $0.1$ here), PGD can easily escape the local maxima even with a large perturbation ratio such as $10$ as shown in Fig. \ref{fig:tr-0.1}. If the input scale is not small enough, escaping from a local maximum will be easier with a smaller perturbation ratio as shown in Fig. \ref{fig:tr-10}. If the input scale is too large, escaping from a local maximum will be difficult even with a small perturbation ratio $0.001$ as shown in Fig. \ref{fig:tr-1000}. These results are consistent with those on the simulated data in Fig. \ref{fig:trajectories}. The experiments are based on a real-world dataset MNIST and a practical multi-layer CNN. More details can be found in the supplementary materials.}
		\label{fig:trajectory-real}
\end{figure*}

In this section, we shed some light on Question 2, which we restate here.

 \textit{When we fix the ratio $r$, do smaller input scales (implying smaller $\varepsilon$) help optimization of adversarial training?}

Our answer to that question is positive --- at least the input scale matters in the initial phase of adversarial training. We experimentally and theoretically answer that question from the perspectives of landscapes and convergence of trajectories.
\subsubsection{Smaller input scales imply more regular landscapes}
In our proofs, the concentration results for all quantities such as $\sup_{\bmdelta \in \bB(\bm 0, \varepsilon)}\|\partial  f(\bma,\bmW,\bmdelta+\bm\bmx) /\partial \bmdelta\|$ and $\min_{\bmdelta,\bmdelta'\in \bB(\bm 0, \varepsilon)}\angle[\partial f(\bmdelta),\partial f(\bmdelta')]$ depend only on the scale of  $\varepsilon$ since in the initial phase, $\bm a$ and $\bmW$ are drawn from initialization distributions which are independent to the inputs. That fact implies with a fixed ratio $r$, a smaller input scale will result in a smaller $\varepsilon$, so as to make all the concentration results hold with a higher probability. Even if the ratio $r$ is large, which means the adversarial attack is more aggressive, the concentration results can hold regardlessly.

Moreover, 
$$\min_{\bmdelta,\bmdelta'}\angle[\partial f(\bmdelta),\partial f(\bmdelta')]\rightarrow 1,$$
as $\varepsilon\rightarrow 0$, which means the angle between $\partial f(\bmdelta)$ and $\partial f(\bmdelta')$ will be very small if $\varepsilon$ is small. Besides, for $\bmdelta\in\varepsilon\bS^{d-1}$, $\bmdelta$ is a local optimum if and only if $\bmdelta$ is parallel to $\partial f(\bmdelta)$. Combining the above facts, it is natural to expect the local minimums will be closer to each other when a smaller $\varepsilon$ is chosen. Actually, there is a threshold $\tau_\varepsilon>0$, when $\varepsilon$ is smaller than $\tau_\varepsilon$, there is only one minimum on the sphere.
\begin{theorem}[\textbf{Informal}]\label{thm:smallepsilon}
Under the settings of Theorem \ref{thm:main}, there exists a threshold $\tau_\varepsilon>0$, such that for $\varepsilon<\tau_\varepsilon$, there is only one local minimum on the sphere with high probability.
\end{theorem}
Theorem \ref{thm:smallepsilon} implies in the initial phase of  adversarial training, a smaller input scale of $\|\bmx\|$ actually can ensure there exists only one single local minimum on the sphere which is also the global minimum. Combined with previous results, the projected gradient descent is able to reach global minimum with high probability.

In  Figure \ref{fig:trajectories}, we can see for a fixed $r$, smaller input scale make the landscape more regular, for instance, the upper left one has only one local minimum. For a large input scale, the landscape will become very complex (see subfigure (i))  unless we use very small perturbation ratio $r$ (see subfigure (g)) .

\begin{figure*}[t]\label{fig:landscape}
		\centering
		\subfigure[Epoch=0]{
			\centering
			\includegraphics[scale=0.3]{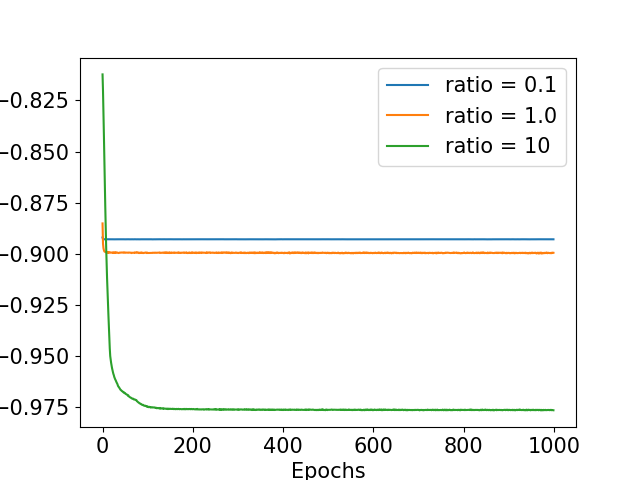}
			\label{fig:tr-epoch-0}}
        \hspace{0.01in}
        \subfigure[Epoch=10]{
		\centering
		\includegraphics[scale=0.3]{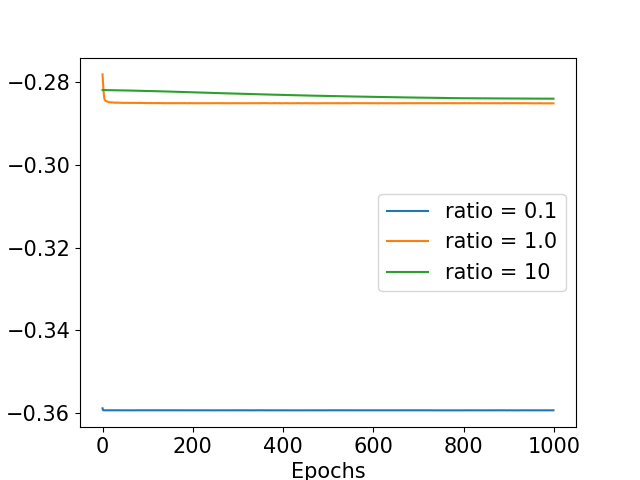}
		\label{fig:tr-epoch-10}}
        \hspace{0.01in}
        \subfigure[Epoch=100]{
		\centering
		\includegraphics[scale=0.3]{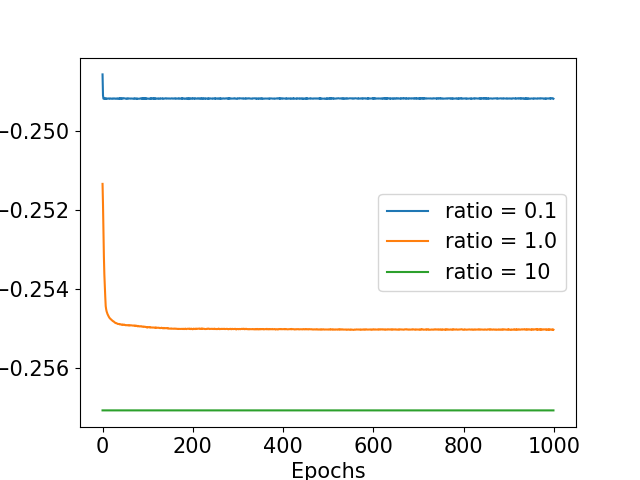}
		\label{fig:tr-epoch-100}}
        \hspace{0.01in}
        \caption{The dynamics of trajectories from local maxima to local minima during the adversarial training process. We use the same setting as that in Figure \ref{fig:trajectory-real} and fix the input scale as $0.1$. After a few epochs of adversarial training, escaping from local maxima is still easy for the small input scale as shown in Fig. \ref{fig:tr-epoch-10}. However, escaping from local maxima will be significantly harder after a lot of training epochs as shown in Fig. \ref{fig:tr-epoch-100}. This influence is more significant on large input scales compared to small ones.}
		\label{fig:trajectory-real-dynamics}
\end{figure*}

\subsubsection{Smaller input scales help convergence} 
Another interesting discovery is inspired by Lemma \ref{lemma:trajectory} in the previous section. If  $\varepsilon$ is not small enough, Eq. (\ref{eq:trajectory}) in Lemma \ref{lemma:trajectory} cannot stand. Thus, when the initial adversary $\bmdelta_0\in\bB^{\circ}(\bm0,\varepsilon)$ is close to one of the local maximums on the sphere, it is possible that the trajectory of projected gradient descent can reach the region $\Delta^{-}_{\eta}$ on the sphere, who contains points close to local maximums. Too close to a local maximum will result in a very small progress in the loss decay at each step, which will take much longer to reach a local minimum. As an illustration, we can see from Figure \ref{fig:trajectory-real} and \ref{fig:trajectory-real-dynamics}, by judging from the decay rate of loss function, we can see a smaller input scale leads to faster loss value decay in the initial phase of adversarial training.

\subsection{General losses}\label{subsec:generalloss}
Previously, we have derived the theory with respect to quadratic loss. In this subsection, we extend the theory to general losses of $(\bmx,y)\in\bR^d\times\bR$ in the following form:
$$L(y, f(\bm\theta, \bmx+\bmdelta )),$$
where we still take $f$ as a two-layer neural network discussed previously:
$$f(\bma,\bmW,\bmdelta+\bm\bmx)=\sum_{r=1}^m a_r\sigma(\bmw_r^T(\bmx+\bmdelta)).$$
Taking derivative with respect to $\bmdelta$:
\begin{equation*}
\frac{\partial L}{\partial \bmdelta}=\frac{\partial L}{\partial f}\cdot \frac{\partial f}{\partial \bmdelta},\quad \frac{\partial^2 L}{\partial \bmdelta^2}= \frac{\partial L}{\partial f} \cdot\frac{\partial^2 f}{\partial \bmdelta^2}+ \frac{\partial^2 L}{\partial f^2}\cdot  \frac{\partial f}{\partial \bmdelta}\Big( \frac{\partial f}{\partial \bmdelta}\Big)^T.
\end{equation*}
Actually the only difference of deriving theory for general losses compared to quadratic losses lies in the different form of $\partial L/\partial \bmdelta$. As long as $\partial L/\partial \bmdelta$ satisfies $\cL<|\partial L/\partial f|<\cU$ for all $\bmdelta\in\bB(\bm0,\varepsilon)$ for some $\cL, \cU>0$, and $|\partial ^2L/\partial f^2|$ is upper bounded by some constant $\cB>0$, all our previous conclusions stand without changing the scale of $\varepsilon$ and $\eta$. Instead of going into too many details, we leave the details to readers who are interested in checking. In the later paragraph, we focus on discussing whether the above assumptions are reasonable. 

Generally,  the loss chosen in the optimization has the following property: $L(y,f)=0$ if and only if $y=f$. The final goal of optimization is to make $L(y,f)$ small and in the initial phase, since we initialize the parameters randomly, we would expect $f(\bm\theta, \bmx+\bmdelta )$ to be ``far from'' the label $y$, in other words, $|L(y,f)|$ is lower bounded by some positive constant $\cL$. Then, by continuity of the loss function, if $\varepsilon$ is small, the change of $|L(y,f)|$ would be expected to be small. As a result, it is reasonable to assume $\partial L/\partial \bmdelta$ satisfies $\cL<|\partial L/\partial f|<\cU$ for all $\bmdelta\in\bB(\bm0,\varepsilon)$ for some $\cL, \cU>0$. Also, with smoothness assumptions on $L$ over $f$, and smoothness assumptions on $f$ over input $\bmx$, since the change of $\varepsilon$ is over a compact set, $|\partial ^2L/\partial f^2|$ should be upper bounded.

We wrap up this subsection with another concrete example besides quadratic loss -- cross entropy loss:
$$L(y, f)=-y\log\left(\frac{\exp(f)}{1+\exp(f)}\right)-(1-y )\log\left(\frac{1}{1+\exp(f)}\right).$$
Then, 
$$\frac{\partial L}{\partial f}=\frac{\exp(f)}{1+\exp(f)}-y,\quad \frac{\partial^2 L}{\partial f^2}=\frac{\exp(f)}{(1+\exp(f))^2}.$$
As discussed above, in the initial phase, we usually have the estimated probability $\exp(f)/(1+\exp(f))$ is not equal to the true probability (here the true probability $y$ is either $0$ or $1$). And with small $\varepsilon>0$, we would expect $\partial L/\partial \bmdelta$ satisfies $\cL<|\partial L/\partial f|<\cU$ for all $\bmdelta\in\bB(\bm0,\varepsilon)$ for some $\cL, \cU>0$. Meanwhile, apparently $0\leq \partial^2 L/\partial f^2\leq 1$. 

\section{Conclusions and Future Work}
In this paper, we theoretically characterize the dynamics of finding adversaries in two-layer fully connected neural networks in the initial phase of training. We also talk about the experimental implications the theory brings.  The main take-away is that in the initial phase of adversarial training, projected gradient method is trustworthy and a smaller input scale can help the adversarial training perform better.

In the future, we hope to extend our theory to higher layer neural networks and to the full dynamics involving weight updates. When considering the full dynamics, as the adversarial training process goes on, the weights become more and more dissimilar to gaussian vectors. Usually, as the adversarial training goes on, $L(y, f)$ will goes to $0$, so we can expect the convergence rate on finding adversaries will be slower since $\frac{\partial L}{\partial \bmdelta}=\frac{\partial L}{\partial f}\cdot \frac{\partial f}{\partial \bmdelta}$ and $\frac{\partial L}{\partial \bmdelta}$ should be close to $0$. 

The landscape of adversaries in the later phase of training will become very complicated due to the intervention of $\bmdelta$ and $\bm\theta$. More importantly, using first order optimization method is possible to result in a cyclic dynamic. It is also interesting to explore how to get rid of the cyclic dynamic problem in the future.

\section{Acknowledgements}
This work is in part supported by NSF award 1763665.

%% file: app_advpaper.tex
\newpage
\appendix
\noindent\textbf{\Large Appendix}
\\

The appendix consists of two parts. Section \ref{app:proof} contains the details for our proof.  Section \ref{sec:experimental-settings} provides more detailed descriptions of our experiments and attaches additional experiments.
\section{Omitted Proofs}\label{app:proof}
We provide a sketch of omitted proofs in this part. For future convenience, we state the expression of the following quantities for $f$:
$$ f(\bma,\bmW,\bmdelta+\bm\bmx) =\sum_{r=1}^ma_r\sigma(\bmw_r^T(\bmx+\bmdelta)).$$
$$\frac{\partial  f(\bma,\bmW,\bmdelta+\bm\bmx) }{\partial \bmdelta}=\sum_{r=1}^ma_r\sigma'(\bmw_r^T(\bmx+\bmdelta))\bmw_r,$$
where $\sigma'(x)=e^x/(1+e^x)$.
$$\frac{\partial^2  f(\bma,\bmW,\bmdelta+\bm\bmx) }{\partial \bmdelta^2}=\sum_{r=1}^m a_r\sigma''(\bmw_r^T(\bmx+\bmdelta))\bmw_r\bmw_r^T,$$
where $\sigma''(x)=e^x/(1+e^x)^2$.
\subsection{Proof of Lemma \ref{lemma:inball}}
\begin{lemma}\label{lemma:fd}
There exists a threshold $m_{\min}=\Omega(d^{5/2})$, so that for each $m>m_{\min}$, there exists $\varepsilon_{\max}(m)=\Theta((\log m)^{-1/2})$, if $\varepsilon<\varepsilon_{\max}$, then with high probability, 
$$ B_l\leq\Big\|\frac{\partial f(\bma,\bmW,\bmdelta+\bm\bmx) }{\partial \bmdelta}\Big\|\leq B_u\sqrt{\log d},$$
for all $\bmdelta\in\bB(\bm0,\varepsilon)$, where $B_l, B_u$ is of order  $\Theta (1)$. 
\end{lemma}
\begin{proof}
We denote $a_r\sigma'(\bmw_r^T(\bmx+\bmdelta))\bmw_r$ as $\bxi_r(\bmdelta)$. For a threshold $T>0$,
$$\partial  f(\bma,\bmW,\bmdelta+\bm\bmx)=\sum_{r=1}^m\bxi_r(\bmdelta)I\{\|\bmw_r\|\leq T\}+\sum_{r=1}^m\bxi_r(\bmdelta)I\{\|\bmw_r\|> T\}.$$
For $t>0$,
\begin{align*}
\{\max_{\bmdelta\in\bB(\bm0,\varepsilon)}\|\partial f(\bma,\bmW,\bmdelta+\bm\bmx)\|>t\}
&\subseteq\Big\{\max_{\bmdelta\in\bB(\bm0,\varepsilon)}\Big\|\frac{1}{\sqrt{m}}\sum_{r=1}^m\bxi_r(\bmdelta)I\{\|\bmw_r\|\leq T\}\Big\|>t/2~~\text{or}~\max_{r}\|\bmw_r\|>T\Big\}
\end{align*}
Let us take $T=c_1\sqrt{\log m}$, as long as $c_1$ is large enough, we know that we can control $\max_{}\bP(\max_{r}\|\bmw_r\|>T)$ to be small. Thus, in order to bound 
$$\bP(\max_{\bmdelta\in \bB(0,\varepsilon)}\|\partial f(\bma,\bmW,\bmdelta+\bm\bmx)\|>t),$$
we only need to bound 
$$\bP\Big(\max_{\bmdelta\in\bB(\bm0,1),\|\bms\|=1}\Big|\sum_{r=1}^m\bms^T[\bxi_r(\varepsilon\bmdelta)-\bxi_r(\bm0)]I\{\|\bmw_r\|\leq T\}\Big|>t/2\Big).$$
We prove $\sqrt{d}\Upsilon(\bmdelta,\bms)$ is a $\psi_2$-process, where
$$\Upsilon(\bmdelta,\bms):=\Big|\sum_{r=1}^m a_r[\sigma'(\bmw_r^T(\bmx+\varepsilon\bmdelta))-\sigma'(\bmw_r^T(\bmx))]\bmw_r^TsI\{\|\bmw_r\|\leq T\}\Big|.$$
We only need to prove
$$\bE\exp(d\frac{|\Upsilon(\bmdelta,\bms)-\Upsilon(\bmdelta',\bms')|^2}{\|\bmdelta-\bmdelta'\|^2+\|\bms-\bms'\|^2})\leq 2.$$
Since  
\begin{align*}
\bE \exp(d\frac{|\Upsilon(\bmdelta,\bms)-\Upsilon(\bmdelta',\bms')|^2}{\|\bmdelta-\bmdelta'\|^2+\|\bms-\bms'\|^2})&=\int_{0}^{\infty}e^t\bP\Big(d\frac{|\Upsilon(\bmdelta,\bms)-\Upsilon(\bmdelta',\bms')|^2}{\|\bmdelta-\bmdelta'\|^2+\|\bms-\bms'\|^2}> t\Big)dt\\
&\leq\int_{0}^{\infty}e^t\bP\Big(|\sum_{r=1}^m\sqrt{d}(u_r+v_r)|>\sqrt{ t}\Big)dt\\
&\leq\int_{0}^{\infty}e^t\bP\Big(|\sum_{r=1}^m\sqrt{d}u_r|>\sqrt{ t}/2\Big)dt+\int_{0}^{\infty}e^t\bP\Big(|\sum_{r=1}^m\sqrt{d}v_r|>\sqrt{ t}/2\Big)dt,
\end{align*}
where 
\begin{align*}
u_r:&= a_r[\sigma'(\bmw_r^T(\bmx+\varepsilon\bmdelta))-\sigma'(\bmw_r^T(\bmx+\varepsilon\bmdelta'))]/\|\bmdelta-\bmdelta'\|(\bmw_r^T\bms)I\{\|\bmw_r\|\leq T\}\\
&=  a_r l(\bmdelta,\bmdelta',\bmw_r,\bmx)\varepsilon\bmw_r^T(\bmdelta-\bmdelta')/\|\bmdelta-\bmdelta'\|\bmw_r^TsI\{\|\bmw_r\|\leq T\}
\end{align*} 
and $|l|$ is bounded by $1$. 
\begin{align*}
v_r:&= a_r[\sigma'(\bmw_r^T(\bmx+\varepsilon\bmdelta'))-\sigma'(\bmw_r^T(\bmx))]\bmw_r^T(\bms-\bms')/\|\bms-\bms'\|I\{\|\bmw_r\|\leq T\}\\
&=  a_r l(\bmdelta',\bm0,\bmw_r,\bmx)\varepsilon\bmw_r^T\bmdelta'\bmw_r^T(\bms-\bms')/\|\bms-\bms'\|I\{\|\bmw_r\|\leq T\}
\end{align*}
Since we take $T=c_1\sqrt{\ln m}$, for $\varepsilon\leq \lambda(\log m)^{-1/2} $, as long as $\lambda$ is small enough, it is easy to see $u_r$ and $v_r$ are sub-gaussian and can ensure 
$$\int_{0}^{\infty}e^t\bP\Big(|\sum_{r=1}^m\sqrt{d}u_r|>\sqrt{ t}/2\Big)dt+\int_{0}^{\infty}e^t\bP\Big(|\sum_{r=1}^m\sqrt{d}v_r|>\sqrt{ t}/2\Big)dt\leq 2.$$

Thus, by chaining, we know, 
\begin{equation}\label{Eq:chaining}
 \max_{\bmdelta,\bmdelta'\in\bB(\bm0,\varepsilon)}\Big\|\frac{\partial f(\bma,\bmW,\bmdelta+\bm\bmx) }{\partial \bmdelta }-\frac{\partial f(\bma,\bmW,\bmdelta'+\bm\bmx) }{\partial \bmdelta}\Big\|\leq B,
 \end{equation}
and as $\lambda\rightarrow0$, $B\rightarrow0$ with high probability.

Now, we prove that 
\begin{equation*}
\bP(\|\frac{1}{\sqrt{m}}\sum_{r=1}^m\bxi_r(\bm0)I\{\|\bmw_r\|\leq T\}\|\geq c_3\sqrt{\log d})
\end{equation*}
with small probability for some constant $c_3>0$.
\begin{align*}
\bP(\|\sum_{r=1}^m\bxi_r(\bm0)I\{\|\bmw_r\|\leq T\}\|\geq t)&\leq d\bP(|\sum_{r=1}^m\bxi^i_r(\bm0)I\{\|\bmw_r\|\leq T\}|\geq \frac{t}{\sqrt{d}})
\end{align*}
where $\bxi^i_r$ is the $i$-th coordinate of $\bxi_r$. By concentration of sub-gaussian, we know when $t=O(\sqrt{\log d})$, the probability is small.

At last, let us provide the lower bound, in which we use central limit theorem. We denote $\bxi_r(\bm0)$ as $\bvarsigma _r$. The covariance matrix of $\bvarsigma_r$
$$\Sigma=E\bvarsigma_r\bvarsigma_r^T.$$

It can be derived directly by using the multi-variate Berry Esseen bound: for any convex set $\cC$, 
\begin{align*}
|\bP(\Sigma^{-1/2}\frac{\partial  f(\bma,\bmW,\bmdelta+\bm\bmx)  }{\partial \bmdelta}\in \cC)-\cN(\bm0,I_d)\{\cC\}|\leq O(d^{1/4})\sum_{r=1}^m \bE\Big\|\frac{\Sigma^{-1/2}}{\sqrt{m}}\bvarsigma_r\Big\|^3.
\end{align*}

Let us take $\cC$ as $\bB(\bm0, \sqrt{cd})$ for $0<c<1$. Then,  by Bernstein inequality, we can obtain
\begin{align*}
\cN(\bm0,I_d)\{\bB(\bm0, \sqrt{(1-c)d})\}&\leq \bP(d-\chi^2(d)\geq cd)\\
&\leq 2e^{\frac{-dc^2}{8}}.
\end{align*}
Thus, plugging in the expression for the gradient
\begin{align*}
\bP(\|\Sigma^{-1/2}\frac{\partial  f(\bma,\bmW,\bmdelta+\bm\bmx)}{\partial \bmdelta}\| \in[  \sqrt{(1-c)d},\sqrt{(1+c)d}])\geq 1-4e^{\frac{-dc^2}{8}}-O(d^{1/4})\sum_{r=1}^m \bE\Big\|\frac{\Sigma^{-1/2}}{\sqrt{m}}\varsigma_r\Big\|^3.
\end{align*}
\end{proof}

\begin{lemma}\label{lemma:fd2}
Under the assumptions of Lemma \ref{lemma:fd}, then with high probability, 
$$\Big\|\frac{\partial^2 f(\bma,\bmW,\bmdelta+\bm\bmx) }{\partial \bmdelta^2}\Big\|\leq M_u\log m,~
\Big\|\frac{\partial^2 f(\bma,\bmW,\bmdelta+\bm\bmx) }{\partial \bmdelta^2}-\frac{\partial^2 f(\bma,\bmW,\bmdelta'+\bm\bmx) }{\partial \bmdelta^2}\Big\|\leq K_u(\log m)^{3/2}\|\bmdelta-\bmdelta'\|,$$
for all $\bmdelta\in\bB(\bm0,\varepsilon)$, where $M_u, K_u$ are of order  $\Theta (1)$. 
\end{lemma}
\begin{proof}
The proof is almost the same as Lemma \ref{lemma:fd}. We will not reiterate it here.
\end{proof}

\paragraph{[Proof of Lemma \ref{lemma:inball}]} 
Under the result of Lemma \ref{lemma:fd}, Lemma \ref{lemma:fd2} and $\cL<|y-f(\bma,\bmW,\bmdelta+\bm\bmx)|<\cU$, notice that when $\bmdelta$ is in the interior of the ball,
$$\bmdelta_{t+1}=\bmdelta_t-\eta\partial L(\bmdelta_t),$$
\begin{equation*}
L(\bmdelta_{t+1})= L(\bmdelta_{t})-\eta\Big\|\frac{\partial L(\bmdelta_t)}{\partial \bmdelta }\Big\|_2^2+\frac{1}{2}\eta ^2\Big(\frac{\partial L(\bmdelta_t)}{\partial \bmdelta}\Big)^T\frac{\partial^2 L(\tilde{\bmdelta}_t)}{\partial \bmdelta^2} \frac{\partial L(\bmdelta_t)}{\partial \bmdelta},
\end{equation*}
for $\tilde{\bmdelta}_t\in \bB(\bm0,\varepsilon)$. 

Since
$$\frac{\partial^2 L(\bmdelta)}{\partial \bmdelta^2}=2(u-y)\frac{\partial^2 f}{\partial \bmdelta^2}+2\frac{\partial f}{\partial \bmdelta}\Big(\frac{\partial f}{\partial \bmdelta}\Big)^T,$$
the dominating term will be 
$$\eta\Big\|\frac{\partial L(\bmdelta_t)}{\partial \bmdelta }\Big\|_2^2,$$
as long as 
$$\frac{1}{\eta}\geq M_u\cU \log m+2B_u^2\log d \geq \max_{\bmdelta\in\bB(\bm0,\varepsilon)}\Big\|\frac{\partial^2 L(\bma,\bmW,\bmdelta+\bm\bmx) }{\partial \bmdelta^2}\Big\|.$$
Since $m=\Omega(d^{5/2})$, $\eta_{\max}(m)=\Theta \big((\log m)^{-1}\big)$, we can obtain the final result easily.
\subsection{Proof of Lemma \ref{lemma:geometry}}
We first provide lemmas for the proof of main theorems of the behavior of projected gradient method on the sphere. The techniques are mainly adopted from \cite{ge2015escaping}, we include them here for completeness. 
\begin{lemma}\label{lemma:g1}
For any $\bmdelta$ and $\bmdelta_0$ on the sphere with radius $\varepsilon$, denoted as $\varepsilon\bS^{d-1}$, let $\cT_0=\cT(\bmdelta_0)$, then
$$\|\cP_{\cT^c_0}(\bmdelta-\bmdelta_0)\|\leq\frac{\|\bmdelta-\bmdelta_0\|^2}{2\varepsilon}.$$
\end{lemma}
Furthermore, if $\|\bmdelta-\bmdelta_0\|<\varepsilon$, we will also have 
$$\|\cP_{\cT^c_0}(\bmdelta-\bmdelta_0)\|\leq\frac{\|\cP_{\cT^c_0}(\bmdelta-\bmdelta_0)\|^2}{\varepsilon}$$
\begin{proof}
Recall that $c(\bmdelta)=\|\bmdelta\|^2-\varepsilon^2$, $\nabla c(\bmdelta)=2\bmdelta\in \cT^c(\bmdelta)$, so we can obtain
\begin{equation}\label{eq:g1}
|\nabla c(\bmdelta_0)^T(\bmdelta-\bmdelta_0)|^2=|2\bmdelta_0^T\cP_{\cT^c_0}(\bmdelta-\bmdelta_0)|^2= 4\varepsilon^2\|\cP_{\cT^c_0}(\bmdelta-\bmdelta_0)\|^2.
\end{equation}

On the other hand, since $c(\bmdelta)=c(\bmdelta)=0$, besides, for all $\bmdelta$ and $\bmdelta_0$ 
$$|c(\bmdelta)-c(\bmdelta_0)-\nabla c(\bmdelta_0)^T(\bmdelta-\bmdelta_0)|=\|\bmdelta-\bmdelta_0\|^2,$$
thus, it results to
\begin{equation}\label{eq:g2}
|\nabla c(\bmdelta_0)^T(\bmdelta-\bmdelta_0)|^2=\|\bmdelta-\bmdelta_0\|^4.
\end{equation}

Combine Eq. (\ref{eq:g1}) with (\ref{eq:g2}), it gives
\begin{equation}
\|\cP_{\cT^c_0}(\bmdelta-\bmdelta_0)\|^2\leq\frac{\|\bmdelta-\bmdelta_0\|^4}{4\varepsilon^2}.
\end{equation}
Notice $\|\bmdelta-\bmdelta_0\|^2=\|\cP_{\cT^c_0}(\bmdelta-\bmdelta_0)\|^2+\|\cP_{\cT_0}(\bmdelta-\bmdelta_0)\|^2$, plugging into Eq. (\ref{eq:g2}), we can obtain 
$$\|\cP_{\cT^c_0}(\bmdelta-\bmdelta_0)\|\leq\frac{\|\cP_{\cT^c_0}(\bmdelta-\bmdelta_0)\|^2}{\varepsilon},~\text{or}~\|\cP_{\cT^c_0}(\bmdelta-\bmdelta_0)\|\geq \varepsilon~(\text{abandoned}).$$
\end{proof}
\begin{lemma}\label{lemma:g2}
For all $\hat{v}\in\cT(\bmdelta)$ and $\hat{w}\in\cT^c(\bmdelta)$, so that $\|\hat{v}\|=\|\hat{w}\|=1$, we have
$$\max\Big\{\|\cP_{\cT^c_0}\cdot\hat{v}\|,~\|\cP_{\cT_0}\cdot\hat{w}\|\Big\}\leq\frac{\|\bmdelta-\bmdelta_0\|}{\varepsilon}.$$
\end{lemma}
\begin{proof}
By Lemma \ref{lemma:g1}, we know 
$$|\cP_{\cT^c_0}\cdot\hat{v}\|=\frac{|2\bmdelta_0^T\hat{v}|}{2\varepsilon}.$$
Besides, since $\hat{v}\in\cT(\bmdelta)$, $2\bmdelta^T\hat{v}=0$, thus,
$$|2\bmdelta_0^T\hat{v}|=|2(\bmdelta_0-\bmdelta)^T\hat{v}|\leq 2\|\bmdelta-\bmdelta_0\|,$$
which gives 
$$\|\cP_{\cT^c_0}\cdot\hat{v}\|\leq\frac{\|\bmdelta-\bmdelta_0\|}{\varepsilon}.$$
Meanwhile, since $\hat{w}\in\cT^c(\bmdelta)$, $\bmdelta^T\hat{w}/\varepsilon=\hat{w}$,
$$\|\cP_{\cT_0}\cdot\hat{w}\|=\|\cP_{\cT^c_0}\cdot\hat{w}-\hat{w}\|=\|\cP_{\cT^c_0}\cdot\hat{w}-\cP_{\cT^c}\cdot\hat{w}\|\leq\|\hat{w}\|\cdot\frac{\|\bmdelta-\bmdelta_0\|}{\varepsilon}=\frac{\|\bmdelta-\bmdelta_0\|}{\varepsilon}.$$
\end{proof}
\paragraph{[Proof of Lemma \ref{lemma:geometry}]}
For any $\hat{v}\in\varepsilon \bS^{d-1}$, let $\tilde{\bmdelta}_1=\bmdelta_0+\eta\hat{v}$ and $\tilde{\bmdelta}_2=\bmdelta_0+\eta\cP_{\cT_0}\cdot\hat{v}$
$$\|\Pi_{\bB(\bm0,\varepsilon)}(\tilde{\bmdelta}_1)-\bmdelta_2\|\leq\frac{4\eta^2}{\varepsilon}.$$
Let $\bmz_1=\Pi_{\bB(\bm0,\varepsilon)}(\tilde{\bmdelta}_1)$, we know that $\|\bmz_1-\tilde{\bmdelta}_1\|\leq \eta$, $(\tilde{\bmdelta}_1-\bmz_1)\in\cT^c(\bmz_1)$ and $\|\bmdelta_0-\bmz_1\|\leq 2\eta$. Thus, by Lemma \ref{lemma:g2}

\begin{equation*}
\|\cP_{\cT_0}(\tilde{\bmdelta}_1-\bmz_1)\|=\frac{\|\cP_{\cT_0}(\tilde{\bmdelta}_1-\bmz_1)\|}{\|\tilde{\bmdelta}_1-\bmz_1\|}\cdot \|\tilde{\bmdelta}_1-\bmz_1\|\leq \frac{\|\bmdelta_0-\bmz_1\|\cdot\|\tilde{\bmdelta}_1-\bmz_1\|}{\varepsilon}\leq\frac{2\eta^2}{\varepsilon}.
\end{equation*}

Let $\bmv_1=\bmdelta_0+\cP_{\cT_0}(\bmz_1-\bmdelta_0)$, then
\begin{equation*}
\|\bmv_1-\tilde{\bmdelta}_2\|=\|\bmv_1-\bmdelta_0-\tilde{\bmdelta}_2+\bmdelta_0\|=\|\cP_{\cT_0}(\bmz_1-\bmdelta_0)-\cP_{\cT_0}(\tilde{\bmdelta}_1-\bmdelta_0)\|\leq\frac{2\eta^2}{\varepsilon}.
\end{equation*}
Meanwhile, by Lemma \ref{lemma:g1},
$$\|\bmz_1-\bmv_1\|=\|\cP_{\cT^c_0}(\bmz_1-\bmdelta_0)\|\leq\frac{\|\bmz_1-\bmdelta_0\|}{2\varepsilon}\leq \frac{2\eta^2}{\varepsilon}.$$
Then, we can obtain
$$\|\bmz_1-\tilde{\bmdelta}_2\|\leq\|\bmv_1-\tilde{\bmdelta}_2\|+\|\bmz_1-\bmv_1\|\leq \frac{4\eta^2}{\varepsilon}.$$
\subsection{Proof of Lemma \ref{lemma:awayopt} and \ref{lemma:nearopt}}
Since
\begin{align*}
\Gamma(\bmdelta)&=2(u-y)\Big[\frac{\partial f(\bma,\bmW,\bmdelta+\bm\bmx) }{\partial \bmdelta}-\varepsilon^{-2}\bmdelta^T\frac{\partial f(\bma,\bmW,\bmdelta+\bm\bmx) }{\partial \bmdelta}\bmdelta\Big].
\end{align*}
\begin{align*}
\Xi(\bmdelta)
&=2(u-y)\Big[\frac{\partial^2 f}{\partial \bmdelta^2}-\varepsilon^{-2}\bmdelta^T\frac{\partial f(\bma,\bmW,\bmdelta+\bm\bmx) }{\partial \bmdelta}I\Big]+2\frac{\partial f}{\partial \bmdelta}\Big(\frac{\partial f}{\partial \bmdelta}\Big)^T,
\end{align*}
Recall with high probability,
$$\Big\|\frac{\partial^2 f(\bma,\bmW,\bmdelta+\bm\bmx) }{\partial \bmdelta^2}\Big\|\leq M_u\log m,~ \Big\|\frac{\partial f(\bma,\bmW,\bmdelta+\bm\bmx) }{\partial \bmdelta}\Big\|\leq B_u\sqrt{\log d},$$
as long as $\angle[\partial f(\bmdelta),\bmdelta]\geq \phi$ for some constant $\phi$  and $\varepsilon$ is small enough, such that
$$\Big|\varepsilon^{-2}\bmdelta^T\frac{\partial f(\bma,\bmW,\bmdelta+\bm\bmx) }{\partial \bmdelta}\Big|\geq\max\Big\{\Big\|\frac{\partial^2 f(\bma,\bmW,\bmdelta+\bm\bmx) }{\partial \bmdelta^2}\Big\|,\frac{1}{\cL}\Big\|\frac{\partial f(\bma,\bmW,\bmdelta+\bm\bmx) }{\partial \bmdelta}\Big\|^2\Big\}$$

the dominating term is 
$$2(u-y)\varepsilon^{-2}\bmdelta^T\frac{\partial f(\bma,\bmW,\bmdelta+\bm\bmx) }{\partial \bmdelta}I.$$
Besides, for any constant $\angle[\partial f(\bmdelta),\bmdelta]\leq \beta,$ notice that $\bmdelta/\varepsilon\in \bS^{d-1}$
$$\|\Gamma(\bmdelta)\|\geq \sqrt{1-\beta^2}\|\partial f(\bmdelta)\|.$$
Combined with the lower bound obtained in Lemma \ref{lemma:fd} $\|f(\bmdelta)\|\geq \cL B_l$, we can obtain the result.
\subsection{Proof of Lemma \ref{lemma:trajectory}}
By Lemma  \ref{lemma:awayopt}, we know if $\angle[\partial f(\bmdelta),\bmdelta]\leq \beta,$
$$\|\Gamma(\bmdelta)\|\geq\sqrt{1-\beta^2}\|\partial f(\bmdelta)\|\geq \cL B_l \sqrt{1-\beta^2}.$$
Thus, if we choose $\beta=\sqrt{1-\eta/(\cL B_l)^2 }$, 
$$\|\Gamma(\bmdelta)\|\geq \sqrt{\eta}.$$
Notice the corresponding solid angle with respect to $\Delta^{-}_{\eta}$ and  $\Delta^{+}_{\eta}$ will be less or equal to $\pi/2$ if we have 
$$\arccos\big(\min_{
\bmdelta,\bmdelta'}\angle[\partial f(\bmdelta),\partial f(\bmdelta')]\big)+\arccos\Big(\sqrt{1-\frac{\eta}{(\cL B_l)^2}}\Big)\leq \frac{\pi}{4}.$$
That is due to the following fact:  $\bmdelta^\star\in\Delta^{+}_{\eta}=\{\bmdelta:\angle[\partial f(\bmdelta),\bmdelta]\geq 1-\sqrt{\eta}/(\cL B_l),\bmdelta\in\varepsilon\bS^{d-1}\}$, where $\bmdelta^\star\parallel \partial f(\bmdelta^\star) $. Then, for any $\bmdelta\in \varepsilon \bS^{d-1}$, 
$$\arccos(\angle[\bmdelta,\bmdelta^\star])\leq  \arccos(\angle[\partial f(\bmdelta^\star),\bmdelta^\star])+ \arccos(\angle[\bmdelta,\partial f(\bmdelta)])+\arccos(\min_{
\bmdelta}\angle[\partial f(\bmdelta),\partial f(\bmdelta^\star)]).$$
 Combined with the fact $\arccos(\angle[\partial f(\bmdelta^\star),\bmdelta^\star])=1$, we know the corresponding solid angle with respect to $\Delta^{+}_{\eta}$ will be less or equal to $\pi/2$. Similar proof can be obtained for $\Delta^{-}_{\eta}$.

Notice if a point $\bmdelta$ in the ball reaches the sphere at $\tilde{\bmdelta}$ by gradient descent, then the tangent direction along longitude at $\tilde{\bmdelta}$ and the direction of $\bmdelta^\star$ should be smaller than the angle between the $\partial f(\bmdelta)$ and  $\partial f(\bmdelta^\star)$. Then, by basic geometry, we know if 
$$\arccos(\min_{
\bmdelta,\bmdelta'}\angle[\partial f(\bmdelta),\partial f(\bmdelta')])< \pi/4,$$
the trajectory initialized by drawing from a uniform distribution over  $\bB^\circ(\bm0,\varepsilon)$ will never reach $\Delta^{-}_{\eta}$. Meanwhile, if there exists $t^*$ such that $\bmdelta_{t^*}\in \Delta^{+}_{\eta}$, then for all $t\geq t^*$, $\bmdelta_{t}\in \Delta^{+}_{\eta}$.

\subsection{Proof of Theorem \ref{thm:main} and Corollary \ref{col:main}}
With the previous results, we are ready to state our main results. Recall for $\bmdelta,\bmdelta_0\in\varepsilon \bS^{d-1}$, if $\partial^2L(\bmdelta,\lambda^*)$ is $\rho$-Lipschitz, that is $\|\partial^2L(\bmdelta_1,\lambda^*)-\partial^2L(\bmdelta_2,\lambda^*)\|\leq\rho\|\bmdelta_1-\bmdelta_2\|$, we can obtain for any $\bmdelta,\bmdelta_0$ are on the sphere, we have 
\begin{equation*}
L(\bmdelta)\leq L(\bmdelta_0)+\Gamma(\bmdelta_0)^T(\bmdelta-\bmdelta_0)+\frac{1}{2}(\bmdelta-\bmdelta_0)^T\Xi(\bmdelta_0)(\bmdelta-\bmdelta_0)+\frac{\rho}{6}\|\bmdelta-\bmdelta_0\|^3.
\end{equation*}
Meanwhile, by Lemma \ref{lemma:geometry}, there exists approximation of PGD: 
\begin{equation*}
\bmdelta_{t+1}=\bmdelta_t-\eta\Gamma(\bmdelta_t)+\tau_t,
\end{equation*}
where $\|\tau_t\|\leq O(\eta^2/\varepsilon)$ for all $t$. Combine the above two formulas, if we further have $\|\Xi\|\leq \nu$, it gives 
\begin{align*}
L(\bmdelta_{t+1})-L(\bmdelta_t)&\leq \Gamma(\bmdelta_t)^T(\bmdelta_{t+1}-\bmdelta_t)+\frac{1}{2}(\bmdelta_{t+1}-\bmdelta_t)^T\Xi(\bmdelta_t)(\bmdelta_{t+1}-\bmdelta_t)+\frac{\rho}{6}\|\bmdelta_{t+1}-\bmdelta_t\|^3\\
&= \Gamma(\bmdelta_t)^T(-\eta\Gamma(\bmdelta_t)+\tau_t)+\frac{1}{2}(-\eta\Gamma(\bmdelta_t)+\tau_t)^T\Xi(\bmdelta_t)(-\eta\Gamma(\bmdelta_t)+\tau_t)+\frac{\rho}{6}\|-\eta\Gamma(\bmdelta_t)+\tau_t\|^3\\
&\leq -\eta\|\Gamma(\bmdelta_t)\|^2+\Gamma(\bmdelta_t)^T\tau_t+\frac{1}{2}\eta^2\Gamma(\bmdelta_t)^T\Xi(\bmdelta_t)\Gamma(\bmdelta_t)+\frac{1}{2}\tau_t^T\Xi(\bmdelta_t)\tau_t-\eta\Gamma(\bmdelta_t)^T\Xi(\bmdelta_t)\tau_t\\
&+\frac{\rho}{6}\big(\eta^3\|\Gamma(\bmdelta_t)\|^3+\|\tau_t\|^3+3\eta^2\|\Gamma(\bmdelta_t)\|^2\cdot\|\tau_t\|+3\eta\|\Gamma(\bmdelta_t)\|\cdot\|\tau_t\|^2\big)\\
&\leq-\eta\|\Gamma(\bmdelta_t)\|^2+\|\Gamma(\bmdelta_t)\|\frac{4\eta^2}{\varepsilon}+\frac{1}{2}\eta^2\nu\|\Gamma(\bmdelta_t)\|^2+\frac{8\eta^4\nu}{\varepsilon^2}+\frac{4\eta^3\nu}{\varepsilon}\|\Gamma(\bmdelta_t)\|\\
&+\frac{\rho}{6}\big(\eta^3\|\Gamma(\bmdelta_t)\|^3+\frac{64\eta^6}{\varepsilon^3}+\frac{12\eta^4}{\varepsilon}\|\Gamma(\bmdelta_t)\|^2+\frac{48\eta^5}{\varepsilon^2}\|\Gamma(\bmdelta_t)\|\big).
\end{align*}
By Lemma \ref{lemma:fd2}, we have $\nu\leq M_u\log m$, $\rho\leq K_u(\log m)^{3/2}$ for some $M_u,K_u$ of order $\Theta(1)$ with high probability under the conditions given. Thus, there exists a threshold $\eta_{\max}(m,\varepsilon)=\min\{\Theta \big((\log m)^{-2}\big),\varepsilon^2\}$, if $\eta\leq\eta_{\max}(m,\varepsilon)$, whenever $\bmdelta_t,\bmdelta_{t+1}\in \varepsilon \bS^{d-1}$ and $\|\Gamma(\bmdelta_t)\|\geq \sqrt{\eta}$,
$$L(\bmdelta_{t+1})-L(\bmdelta_t)\leq -\Omega(\eta^2).$$
We conclude the above statements by the following lemma. 
 \begin{lemma}\label{lemma:progressonsphere}
 If $m=\Omega(d^{5/2})$, there exists $\varepsilon_{\max}(m)=\Theta \big((\log m)^{-1/2}\big)$ and $\eta_{\max}(m,\varepsilon)=\min\{\Theta \big((\log m)^{-2}\big),\varepsilon^2\}$, if $\varepsilon<\varepsilon_{\max}(m)$, $\eta<\eta_{\max}(m,\varepsilon)$
 $$\Big\|\frac{\partial f(\bma,\bmW,\bmdelta+\bm\bmx) }{\partial \bmdelta}\Big\|\leq B_u\sqrt{\log d},~\Big\|\frac{\partial^2 f(\bma,\bmW,\bmdelta+\bm\bmx) }{\partial \bmdelta^2}\Big\|\leq M_u\log m,$$
 $$\Big\|\frac{\partial^2 f(\bma,\bmW,\bmdelta+\bm\bmx) }{\partial \bmdelta^2}-\frac{\partial^2 f(\bma,\bmW,\bmdelta'+\bm\bmx) }{\partial \bmdelta^2}\Big\|\leq K_u(\log m)^{3/2}\|\bmdelta-\bmdelta'\|,$$
for all $\bmdelta\in\bB(\bm0,\varepsilon)$, where $M_u, K_u$ are of order  $\Theta (1)$ with high probability, besides whenever $\bmdelta_t\in \varepsilon\bS^{d-1}$ and $\|\Gamma(\bmdelta_t)\|\geq \sqrt{\eta}$,
$$L(\bmdelta_{t+1})-L(\bmdelta_t)\leq -\Omega(\eta^2).$$
\end{lemma}
\begin{proof}
By the nature of PGD, we know $\bmdelta_{t+1}$ is either in the ball $\bB(\bm0,\varepsilon)$ or on the sphere $\varepsilon\bS^{d-1}$. When $\bmdelta_{t+1}\in\varepsilon\bS^{d-1}$, by the analysis above, we have $L(\bmdelta_{t+1})-L(\bmdelta_t)\leq -\Omega(\eta^2).$ When $\bmdelta_{t+1}\in \bB^{\circ}(\bm0,\varepsilon)$, by Lemma \ref{lemma:inball}, 
$$L(\bmdelta_{t+1})-L(\bmdelta_t)\leq -\Omega(\eta).$$
Thus, to sum up, we have for both cases
$$L(\bmdelta_{t+1})-L(\bmdelta_t)\leq -\Omega(\eta^2).$$
\end{proof}
Then, we need to deal with the case when $\|\Gamma(\bmdelta_t)\|\leq \sqrt{\eta}$. If we further denote the region $\Lambda^{+}=\{\bmdelta:\bmv^T\Xi(\bmdelta)\bmv\geq \gamma,\bmdelta\in\varepsilon\bS^{d-1}\}$ and $\Lambda^{-}=\{\bmdelta:\bmv^T\Xi(\bmdelta)\bmv\leq-\gamma,\bmdelta\in\varepsilon\bS^{d-1}\}$, where $\gamma$ is the universal constant specified in Lemma \ref{lemma:nearopt}. By Lemma \ref{lemma:awayopt} and \ref{lemma:nearopt}, we know if $\sqrt{\eta}\leq \cL B_l \sqrt{1-\phi^2}$, we have $\Delta^{+}_{\eta}\subseteq \Lambda^{+}$ and $\Delta^{-}_{\eta}\subseteq \Lambda^{-}$. $\Delta^{+}_{\eta}$ are those points near local optimums, which we try to avoid being stuck at. Lemma \ref{lemma:trajectory} provides insights how can the trajectory avoids being stuck near the local optimums.

The following corollary can help us realize Lemma \ref{lemma:trajectory}. Specifically, by zooming into the proof of Lemma \ref{lemma:inball}, it is straightforward to obtain the following corollary.
\begin{corollary}\label{coro:angle}
For any $\bmdelta,\bmdelta'\in\bB(\bm0,\varepsilon)$
$$\angle[\partial f(\bmdelta),\partial f(\bmdelta')]\rightarrow 1,$$
as $\varepsilon\rightarrow 0$ in  Eq. \ref{Eq:chaining} under the setting of Lemma \ref{lemma:fd} .
\end{corollary}
\begin{proof}
Once we notice that
\begin{align*}
\angle [\partial f(\bmdelta),\partial f(\bmdelta')]-1&=\Big\langle\frac{\partial f(\bmdelta)}{\|\partial f(\bmdelta)\|}-\frac{\partial f(\bmdelta')}{\|\partial f(\bmdelta')\|},\frac{\partial f(\bmdelta')}{\|\partial f(\bmdelta')\|}\Big \rangle\\
&\leq \frac{2\|\partial f(\bmdelta)-\partial f(\bmdelta')\|}{\|\partial f(\bmdelta)\|}
\end{align*}
and $\|\partial f(\bmdelta)\|\geq\cL B_l$, the proof is straightforward.
\end{proof} 
Thus, if $\eta,\varepsilon$ are smaller than some constant thresholds, then 
\begin{equation*}
\arccos\big(\min_{
\bmdelta,\bmdelta'}\angle[\partial f(\bmdelta),\partial f(\bmdelta')]\big)+\arccos\Big(\sqrt{1-\frac{\eta}{(\cL B_l)^2}}\Big)\leq \frac{\pi}{4}
\end{equation*}
stands. As a result, the trajectory can successfully avoid being stuck near local maximums.

The only case left is when $\|\Gamma(\bmdelta_t)\|\leq \sqrt{\eta}$ and $\bmdelta_t\in \Delta^{-}_\eta$. If $\bmdelta^*$ is one of the local minimums and we focus on studying the case when $\bmdelta_t$ falls at the local neighborhood belongs to $ \Delta^{-}_\eta $ corresponding to $\bmdelta^*$. 

Notice that
\begin{equation*}
\Gamma(\bmdelta_t)=\Gamma(\bmdelta^*)+\int_{0}^1\nabla\Gamma(\bmdelta^*+t(\bmdelta_t-\bmdelta^*))dt\cdot(\bmdelta_t-\bmdelta^*).
\end{equation*}
By looking up the derivative of $\Gamma(\bmdelta)$, we have the following characterization:
\begin{align}
\nabla\Gamma(\bmdelta)=\Xi(\bmdelta)-\nabla c(\bmdelta)\nabla \lambda^*(\bmdelta)^T.
\end{align}

Denote
$$N(\bmdelta)=-\nabla c(\bmdelta)\nabla \lambda^*(\bmdelta)^T.$$
If $\bmdelta\in\varepsilon \bS^{d-1}$, $\nabla c(\bmdelta)=2\bmdelta$ is parallel to the normal vector of the tangent space at $\bmdelta$, thus, we have for any $\bmv$, $N(\bmdelta)\bmv\in \cT^c(\bmdelta)$.

The fact that $\nabla\Gamma(\bmdelta)=\Xi(\bmdelta)+N(\bmdelta)$ is very important, since it gives the following extra characterization of  $\bmdelta_t$: for small enough $\eta$, if $\|\Gamma(\bmdelta_t)\|\leq \sqrt{\eta}$, then $\|\bmdelta_t-\bmdelta^*\|=O(\sqrt{\eta})$. Besides, since $\Gamma(\bmdelta^*)=0$, if $\|\Gamma(\bmdelta_t)\|\leq \sqrt{\eta}$, $\bmdelta_t$ falls into a neighborhood of $\bmdelta^*$ such that $\Xi$ has smallest eigenvalue larger or equal than $\gamma>0$ as long as $\eta$ is small enough.  Besides, we have when $\|\bmdelta_t-\bmdelta^*\|=\Omega(\sqrt{\eta})$, then $\|\Gamma(\bmdelta_t)\|=\Omega(\sqrt{\eta})$.
The previous discussion can be formalized as the following lemma.
\begin{lemma}\label{lemma:innerproduct}
For small enough $\eta$, if $\|\Gamma(\bmdelta_t)\|\leq \sqrt{\eta}$, and $\bmdelta_t$ is in the neighborhood of $\bmdelta^*$ such that $\Xi$ has smallest eigenvalue larger or equal than $\gamma$,
then 
$$\|\bmdelta_t-\bmdelta^*\|=O(\sqrt{\eta}).$$
Furthermore, we have
$$\Gamma(\bmdelta_t)^T(\bmdelta_t-\bmdelta^*)\geq\frac{\gamma}{2}\|\bmdelta_t-\bmdelta^*\|^2.$$
\end{lemma}
\begin{proof}
By Lemma \ref{lemma:g1}, we know that for $\bmdelta_t,\bmdelta^*\in\varepsilon\bS^{d-1}$
$$\|\cP_{\cT^c_{\bmdelta_t}}(\bmdelta_t-\bmdelta^*)\|\leq\frac{\|\bmdelta_t-\bmdelta^*\|^2}{2\varepsilon}.$$
As $\bmdelta^*$ is one of the minimizer on the sphere, we must have $\Gamma(\bmdelta^*)=0$. Thus, for small enough $\alpha$, if we denote $\bmdelta^{\cP}_{\bmv}=\Pi_{\bB(\bm0,\varepsilon)}(\bmdelta+\alpha\bmv)$, for any $\bmv \in \cT(\bmdelta)$ 
with norm $l$,
\begin{align*}
\|\Gamma(\bmdelta^{\cP}_{\bmv})-\Gamma(\bmdelta^*)\|&\approx\|\nabla\Gamma(\bmdelta^*)(\bmdelta^{\cP}_{\bmv}-\bmdelta^*)\|\\
&\geq \|\Xi(\bmdelta^*)(\bmdelta^{\cP}_{\bmv}-\bmdelta^*)\|-\|N(\bmdelta^*)(\bmdelta^{\cP}_{\bmv}-\bmdelta^*)\|\\
&\geq \gamma \alpha-\frac{4\alpha^2}{\varepsilon}(\|\Xi\|+\|N\|)\\
&\geq \frac{ \gamma \alpha}{2}.
\end{align*}
If we further denote $\cR$ as the region where $\bmdelta$ has the following properties:
\begin{itemize}
\item$\text{smallest eigenvalue of }~\Xi(\bmdelta)~ \text{larger or equal to}~\gamma$;
\item$\text{the distance from }~\bmdelta~\text{to one of the minimizers is at least}~\Omega(\eta),$
\end{itemize}
with abuse of notations, we want to prove for any $\bmdelta$ belongs to $\cR$, there is a path $\{\bmdelta_t\}$ to the region that $\text{the distance from }~\bmdelta~\text{to one of the minimizers is at most}~O(\eta),$ where $\bmdelta_0=\bmdelta$, such that $\|\Gamma(\bmdelta_t)\|$ is decreasing along the path. If that statement is true let $\alpha=c\sqrt{\eta}$ for some constant $c$, and $\eta/\varepsilon^4$ is small enough, then we know $\|\Gamma(\bmdelta)\|\leq \sqrt{\eta}$ implies $\bmdelta$ being very close to one of the minimizers, distance up to $O(\sqrt{\eta})$. So
\begin{align*}
\Gamma(\bmdelta_t)^T(\bmdelta_t-\bmdelta)&=(\bmdelta_t-\bmdelta^*)^T\int_{0}^1\nabla\Gamma(\bmdelta^*+t(\bmdelta_t-\bmdelta^*))dt\cdot(\bmdelta_t-\bmdelta^*)\\
&\geq \gamma\|\bmdelta_t-\bmdelta^*\|^2-O(\|\bmdelta_t-\bmdelta^*\|^3)\\
&\geq \frac{\gamma}{2}\|\bmdelta_t-\bmdelta^*\|^2.
\end{align*}

Finally, we show we can always find such path that of decreasing norm of $\Gamma$. Notice
\begin{align*}
\frac{d \|\Gamma(\bmdelta)\|^2}{2dt}&=\langle\Gamma(\bmdelta_t),\Xi\frac{d\bmdelta_t}{dt}\rangle.
\end{align*}

If $d\bmdelta_t/dt=- \varpi \Gamma(\bmdelta_t)$ for some constant $ \varpi$, then $\langle\Gamma(\bmdelta_t),\Xi\frac{d\bmdelta_t}{dt}\rangle\leq -\varpi \gamma\|\Gamma(\bmdelta_t)\|^2$, which implies the norm is decreasing along the path. Thus, for discrete version $\bmdelta_{t+1}=\Pi_{\bB(\bm0,\varepsilon)}(\bmdelta_t-\varpi \nabla L(\bmdelta_t))$, as long as $\varpi$ is small enough (can be much smaller than $\eta$), combine with Lemma \ref{lemma:geometry}, the dominating term of $  d\|\Gamma(\bmdelta_t))\|^2/dt$ will be negative, thus, the proof is complete.

\end{proof}
Now we are ready to state the convergence result when $\bmdelta_t\in\varepsilon\bS^{d-1}$ is in a neighborhood of one of the minimizers $\bmdelta^*$. We have shown that eventually for some $t$, $\|\bmdelta_t-\bmdelta^*\|\leq O(\sqrt{\eta})$. We need further to show that the trajectory will remain near $\bmdelta^*$ ever since.
\begin{lemma}\label{lemma:convergenceinneighbourhood}
For small enough $\eta$, for a $T>0$ such that $\|\Gamma(\bmdelta_T)\|\leq \sqrt{\eta}$, and is in the neighborhood of a minimizer $\bmdelta^*$ and $\Xi$ has smallest eigenvalue larger or equal than $\gamma$, for any $t\geq T$,
$$\|\bmdelta_t-\bmdelta^*\|\leq O(\sqrt{\eta}).$$
\end{lemma}
\begin{proof}
Notice that if $\|\bmdelta_t-\bmdelta^*\|\leq c\sqrt{\eta}$ for some constant $c>0$,
\begin{align*}
\|\bmdelta_{t+1}-\bmdelta^*\|^2&=\|\bmdelta_{t}-\eta\Gamma(\bmdelta_t)+\iota_t-\bmdelta^*\|^2\\
&=\|\bmdelta_t-\bmdelta^*\|^2-2\eta\Gamma(\bmdelta_t)^T(\bmdelta_t-\bmdelta^*)+2\iota_t^T(\bmdelta_t-\bmdelta^*)+\|\eta\Gamma(\bmdelta_t)-\iota_t\|^2\\
&\leq (1-\gamma\eta)\|\bmdelta_t-\bmdelta^*\|^2+\|2\iota_t\|\|\bmdelta_t-\bmdelta^*\|+2c\eta^3+\frac{32\eta^4}{\varepsilon^2}\\
&\leq (1-\gamma\eta)\|\bmdelta_t-\bmdelta^*\|^2+\frac{8\eta^{2.5}}{\varepsilon}+o(\eta^2).
\end{align*}
Then, $\|\bmdelta_{t+1}-\bmdelta^*\|\leq\sqrt{\eta}$ for small enough $\eta$ and $\eta^{0.5}/\varepsilon=o(1)$. Further, 
$$\|\bmdelta_{t+1}-\bmdelta^*\|^2-\frac{9\eta}{\gamma}\leq (1-\gamma\eta)(\|\bmdelta_{t}-\bmdelta^*\|^2-\frac{9\eta}{\gamma}).$$
Then, the proof is straightforward.
\end{proof}
\paragraph{Shrinking step size $\eta_t$} The above discussions are all about constant $\eta$. Now, we further discuss about shrinking step size. Specifically, after $\|\Gamma(\bmdelta_t)\|$ reaches $\sqrt{\eta}$, we can shrink the learning rate with suitable $\eta_0<\eta$, and $\{\eta_s\}_{s\geq 0}$ are strictly decreasing with respect to $s\geq 0$, such that 
$$\bmdelta_{t+s+1}=\Pi_{\bB(\bm0,\varepsilon)}\big[\bmdelta_{t+s}-\eta_s\nabla L(\bmdelta_{t+s})\big]$$
By Lemma \ref{lemma:convergenceinneighbourhood}, for small enough $\eta$ and $\eta^{0.5}/\varepsilon=o(1)$, notice that $\|\bmdelta_{t+s}-\bmdelta^*\|\leq \varepsilon$, we can still have 
$$\|\bmdelta_{t+1+s}-\bmdelta^*\|^2\leq (1-\gamma\eta_s)\|\bmdelta_{t+s}-\bmdelta^*\|^2+9\eta_s^{2}.$$
For simplicity, we denote $\|\bmdelta_{t+s}-\bmdelta^*\|^2$ as $D_s$. We would show if $\eta_s\rightarrow 0$, $D_s\rightarrow 0$.
\begin{lemma}\label{lemma:shrinkingstepsize}
As long as  $\eta_s\rightarrow 0$ and  $\Pi_{i=0}^{k}(1-\gamma\eta_i/2)\rightarrow 0$, we can obtain $D_s\rightarrow 0$. Furthermore, if 
\begin{equation*}
\eta_s\Pi_{i=0}^k(1-\frac{\gamma\eta_{s+i}}{2})\leq \eta_{s+k+1}
\end{equation*}
for all $s,k\in\bN$, then for all $s\in \bN$, 
$$D_s\leq O(\eta_s).$$
Specifically, if $\eta_s=1/(s+z)$ for  large enough integer $z$, 
$$D_s\leq O(\frac{1}{z+s}).$$
\end{lemma}
\begin{proof}
First, if there exists $S\in \bN^+$, such that for all $s\geq S$, we have 
$$D_s\geq \frac{18\eta_s}{\gamma}, $$
then we have 
$$D_{s+1}\leq (1-\gamma\eta_s)D_s+9\eta_s^{2}\leq (1-\frac{\gamma\eta_s}{2})D_s.$$
As a result, $\forall s\geq S$, and $k\in \bN^{+}$
$$D_{s+k+1}\leq \Pi_{i=s}^{s+k}(1-\frac{\gamma\eta_i}{2})D_s.$$
On the other hand, if  there does not exist such $S$, there exists infinitely many $s$, such that 
$$D_s<\frac{18\eta_s}{\gamma}.$$
Moreover, if $D_s<18\eta_s/\gamma$
\begin{align*}
D_{s+1}&\leq (1-\gamma\eta_s)D_s+9\eta_s^{2}\\
&\leq  (1-\gamma\eta_s)\frac{18\eta_s}{\gamma}+9\eta_s^{2}\\
&\leq \frac{18\eta_s}{\gamma}.
\end{align*}

So, for any $\varepsilon>0$, we can choose large enough $s$, such that $D_s< \frac{18\eta_s}{\gamma}$, and $D_{s+1+i}<\varepsilon$ for any $i\geq 0$. So, we will always have 
$D_t\rightarrow 0$ as $t\rightarrow \infty$ as long as $\Pi_{i=0}^{k}(1-\gamma\eta_i/2)\rightarrow_{k\rightarrow\infty} 0$ and $\eta_k\rightarrow_{k\rightarrow\infty}0$.

Besides, we have 
$$D_{s+1}\leq \max\Big\{ (1-\frac{\gamma\eta_s}{2})D_s, \frac{18\eta_s}{\gamma}\Big\}.$$
Notice if 
\begin{equation}\label{eq:condition}
\eta_s\Pi_{i=0}^k(1-\frac{\gamma\eta_{s+i}}{2})\leq \eta_{s+k+1}
\end{equation}
for all $s,k\in\bN$, then for all $s\in \bN$, we can obtain
$$D_{s}\leq \frac{18\eta_s}{\gamma}.$$
For example, if $\eta_s=2/(\gamma s+\gamma z)$ for  large enough integer $z$, Eq. \ref{eq:condition} is satisfied by simple algebra and we have
$$D_s\leq O(\frac{1}{z+s}).$$

\end{proof}

Now we are ready to state the proof of our main theorem.
\paragraph{[Proof of Theorem \ref{thm:main} and and Corollary \ref{col:main}]}
Under the assumptions, we have the following properties hold simultaneously:
\begin{itemize}
\item[a]. $$ B_l\leq\Big\|\frac{\partial f(\bma,\bmW,\bmdelta+\bm\bmx) }{\partial \bmdelta}\Big\|\leq B_u\sqrt{\log d},$$
for all $\bmdelta\in\bB(\bm0,\varepsilon)$, where $B_l, B_u$ is of order  $\Theta (1)$. 
\item[b]. $$\Big\|\frac{\partial^2 f(\bma,\bmW,\bmdelta+\bm\bmx) }{\partial \bmdelta^2}\Big\|\leq M_u\log m~,\Big\|\frac{\partial^2 f(\bma,\bmW,\bmdelta+\bm\bmx) }{\partial \bmdelta^2}-\frac{\partial^2 f(\bma,\bmW,\bmdelta'+\bm\bmx) }{\partial \bmdelta^2}\Big\|\leq K_u(\log m)^{3/2}\|\bmdelta-\bmdelta'\|,$$
for all $\bmdelta\in\bB(\bm0,\varepsilon)$, where $M_u, K_u$ are of order  $\Theta (1)$. 
\item[c.] $$\arccos\big(\min_{
\bmdelta,\bmdelta'}\angle[\partial f(\bmdelta),\partial f(\bmdelta')]\big)+\arccos\Big(\sqrt{1-\frac{\eta}{(\cL B_l)^2}}\Big)\leq \frac{\pi}{4}.$$
\item[d.]There exists a threshold $\varepsilon_{\max}(m)=\Theta((\log m)^{-1})$, if $\varepsilon<\varepsilon_{\max}$, with high probability, there exists universal constants $\phi,\gamma>0$, for any  $\bmdelta\in\varepsilon\bS^{d-1}$, such that
$$\angle[\partial f(\bmdelta),\bmdelta]\geq \phi,$$
then for all $\|\bmv\|=1$,
$$\text{sgn}\Big((y-u)\bmdelta^T\partial f(
\bmdelta)\Big)\cdot\bmv^T\Xi\bmv\geq \gamma.$$
\end{itemize}
As a result, whenever $\bmdelta_{t+1}\in\bB^\circ(\bm0,\varepsilon)$
\begin{equation*}
L(\bmdelta_{t+1})-L(\bmdelta_t)\leq -\Omega(\eta).
\end{equation*}
Whenever $\bmdelta_t\in\varepsilon\bS^{d-1}$,  condition $c$ will ensure $\bmdelta_t\notin \Delta^{-}_{\eta}$ so that he trajectory will not stuck near local maximums. Besides, we have $\|\Gamma(\bmdelta)\|\geq \sqrt{\eta}$ if $\bmdelta\notin \Delta^{+}_{\eta}$ for $\bmdelta\in\varepsilon\bS^{d-1}$. That can ensure $L(\bmdelta_{t+1})-L(\bmdelta_{t})\leq-\Omega(\eta^2)$ for $\bmdelta_{t},\bmdelta_{t+1}\in \varepsilon\bS^{d-1}$. 

From the above discussions, we divide the ball into three regions. Let $\cR_1=\bB^{\circ}(\bm0,\varepsilon)$ be the interior of the ball. Let $\cR_2=\Delta^{+}_\eta$ and $\cR_3=\bB(\bm0,\varepsilon)\cap (\cR_1\cup \cR_2)^c$. Since there exists $\cL, \cU>0$ such that $\cL<|y-f(\bma,\bmW,\bmdelta+\bm\bmx)|<\cU$ for all $\bmdelta\in\bB(\bm0,\varepsilon)$, we claim at most $O(\eta^{-2})$ iterations, the trajectory will arrive at $\cR_2$. That is because each step will have at least $O(\eta^2)$  progress in decreasing the value of loss if  $\bmdelta_t\notin \cR_2$.

 Lastly, when $\|\Gamma(\bmdelta)\|\leq \sqrt{\eta}$, the results follow by applying Lemma \ref{lemma:convergenceinneighbourhood} and \ref{lemma:shrinkingstepsize}.

\subsection{Proof of Theorem \ref{thm:smallepsilon}}

\paragraph{Formal Statement of Theorem \ref{thm:smallepsilon}} Recall from Lemma \ref{lemma:nearopt}, for $m=\Omega(d^{5/2})$, there exists a threshold $\varepsilon_{\max}(m)=\Theta((\log m)^{-1})$, if $\varepsilon<\varepsilon_{\max}$, with high probability, there exists universal constants $\phi,\gamma>0$, for any  $\bmdelta\in\varepsilon\bS^{d-1}$, such that
$$\angle[\partial f(\bmdelta),\bmdelta]\geq \phi,$$
then for all $\|\bmv\|=1$,
$$\text{sgn}\Big((y-u)\bmdelta^T\partial f(
\bmdelta)\Big)\cdot\bmv^T\Xi\bmv\geq \gamma.$$
Based on that, there also exists a threshold $\tau_\varepsilon>0$, such that when $\varepsilon < \tau_\varepsilon$,
$$\min_{
\bmdelta,\bmdelta'}\angle[\partial f(\bmdelta),\partial f(\bmdelta')]\big)\geq \phi,$$
and there is only one minimum on the sphere in that case.
\begin{proof}
Notice 
$\min_{
\bmdelta,\bmdelta'}\angle[\partial f(\bmdelta),\partial f(\bmdelta')]\big)\rightarrow 1$ as $\varepsilon\rightarrow 0$, so we know if $\varepsilon$ is small enough, we will have $\min_{
\bmdelta,\bmdelta'}\angle[\partial f(\bmdelta),\partial f(\bmdelta')]\big)\geq \phi.$ That means the solid cone formed by $\partial f(\bmdelta)$ is included in the corresponding solid cone of $\Lambda^{+}=\{\bmdelta:\bmv^T\Xi(\bmdelta)\bmv\geq \gamma,\bmdelta\in\varepsilon\bS^{d-1}\}$.

Assume there are two local minimums, actually in $\Lambda^{+}$ the local minimums are strict , then there exists a path on the sphere such that there is a local maximum on this path. However, that is impossible since the Hessian approximate is positive definite.
\end{proof}

\section{More About Experiments}\label{sec:experimental-settings}
\subsection{Implementation Details}
\label{subsec:implementation-details}
{\bf Loss landscapes on simulated data.} In our experiments, we use a two-layer neural network with the hidden size of $16$ and the initialization is as Sec. \ref{subsec:setup}.  We first randomly choose an two-dimensional input $x$ with a norm smaller than the input scale. The epsilon $\epsilon$ is the product of the perturbation ratio $r$ and the input scale.  We then randomly choose $10000$ perturbations in the epsilon ball. The adversarial losses of these perturbations on the input $x$ are shown in Figure \ref{fig:trajectories}. The choice of the input $x$ is not important in our experiments and the landscapes based on another random choice is shown in  Sec. \ref{subsec:additional-results}. The impact of the width of the hidden layer is also shown in Sec. \ref{subsec:additional-results}. \\
{\bf Trajectories on simulated data.} We use the same settings for neural networks as those in the landscapes. We choose the perturbation with the maximal loss among the $10000$ random sampled perturbations as our local maxima. To show the trajectories, we conduct PGD $10$ times with the best learning rate from 1e$-6$, 1e$-5$, 1e$-4$, 1e$-3$, 1e$-2$, 1e$-1$, 1, 1e
$1$, 1e$2$, 1e$3$. \\
{\bf Trajectories on real-world data.} Our experiments in Fig. \ref{fig:trajectory-real} are based on a real-world dataset MNIST. We use the same multiple-layer CNN architecture except the dropout in \url{https://github.com/pytorch/examples/tree/master/mnist}. We change the original ten-class classification to binary classification to distinguish odd and even numbers.  Because the inputs are high-dimensional ($28 \times 28$), we instead show the loss of PGD from the local maxima to the local minima. We first randomly sample an image from MNIST as our input $x$. We then start with a random perturbation and use PGD to find the local maxima. After tuning the hyperparameters, we find that running $1000$ epochs of PGD with a learning rate of $1.0$ can achieve good enough local maxima. After that, we run $1000$ epochs of PGD with a learning rate of $1.0$ to show the adversarial loss of each point on the trajectory from the local maxima to the local minima.  \\
{\bf Dynamics of trajectories on real-world data.} We use the same setting as that in Fig. \ref{fig:trajectory-real}. To train the CNN model, we randomly sample $100$ images with odd numbers and $100$ figures with even numbers from MNIST as our training data. We set the learning rate of adversarial training as $0.01 \times r$, where $r$ is the perturbation ratio.  
\subsection{Additional Results}
\label{subsec:additional-results}
\begin{figure*}[t]
		\centering
		
		\subfigure[Input scale=0.01, hidden size=16]{
			\centering
			\includegraphics[scale=0.3]{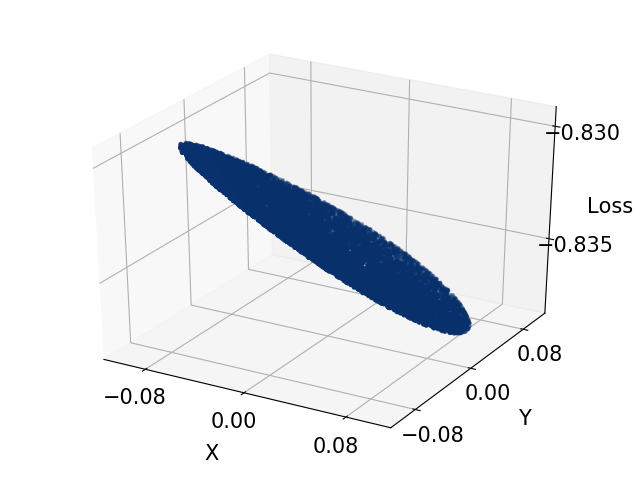}
			\label{fig:l-0.01-10-16}}
        \hspace{0.01in}
        \subfigure[Input scale=0.01, hidden size=128]{
		\centering
		\includegraphics[scale=0.3]{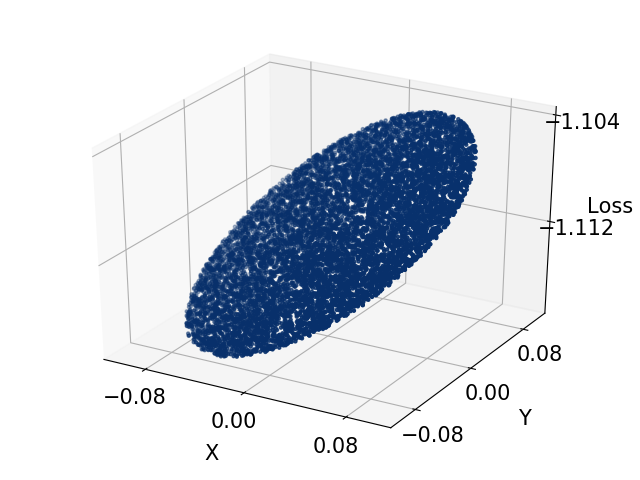}
		\label{fig:l-0.01-10-128}}
        \hspace{0.01in}
        \subfigure[Input scale=0.01, hidden size=1024]{
		\centering
		\includegraphics[scale=0.3]{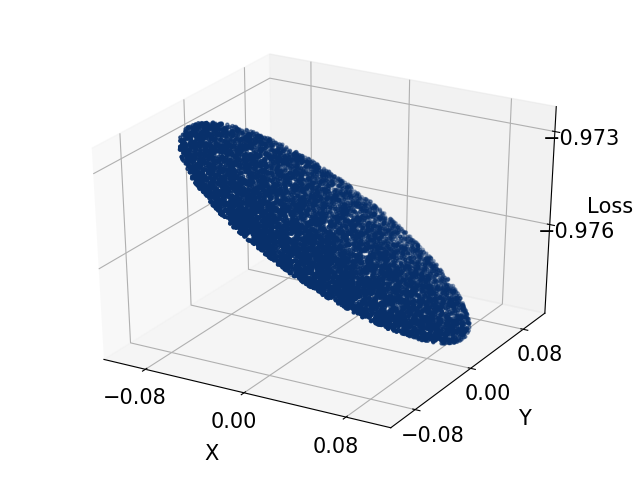}
		\label{fig:l-0.01-10-1024}}
        \hspace{0.01in}
        
        \subfigure[Input scale=1.0, hidden size=16]{
			\centering
			\includegraphics[scale=0.3]{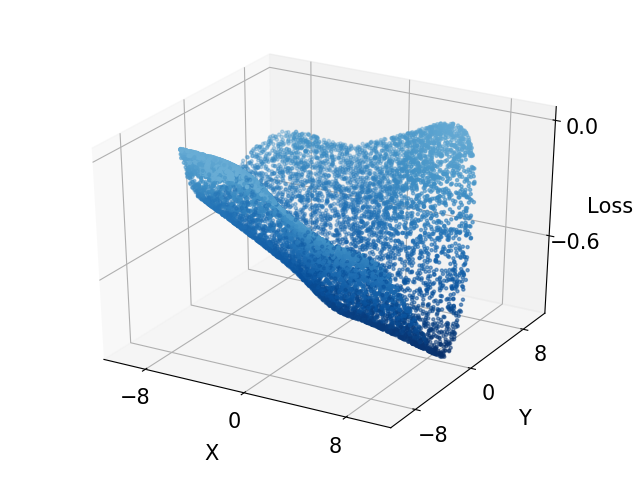}
			\label{fig:l-1.0-10-16}}
        \hspace{0.01in}
        \subfigure[Input scale=1.0, hidden size=128]{
		\centering
		\includegraphics[scale=0.3]{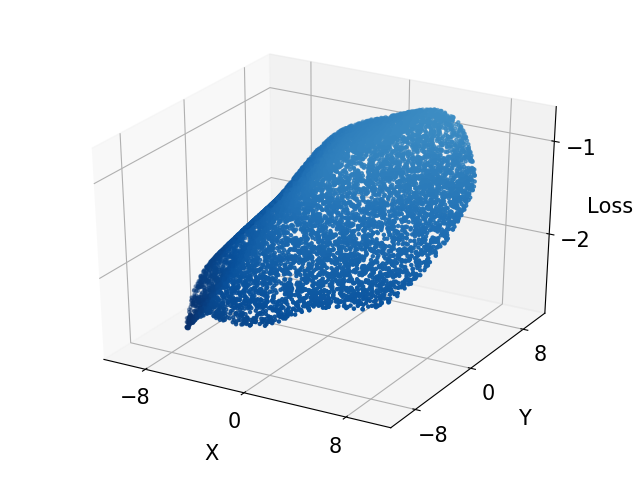}
		\label{fig:l-1.0-10-128}}
        \hspace{0.01in}
        \subfigure[Input scale=1.0, hidden size=1024]{
		\centering
		\includegraphics[scale=0.3]{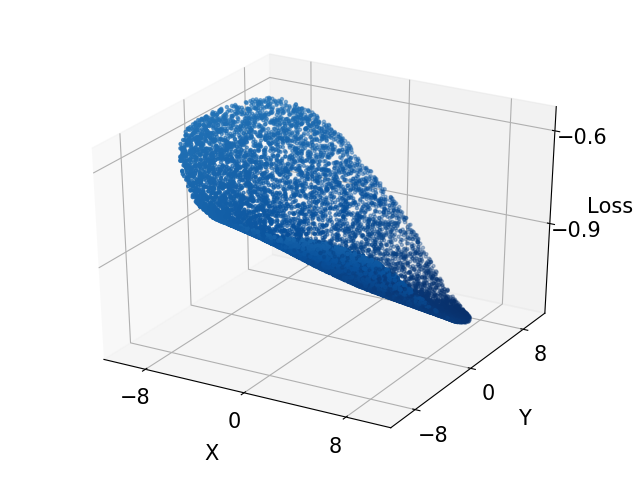}
		\label{fig:l-1.0-10-1024}}
        \hspace{0.01in}
     
        \subfigure[Input scale=100, hidden size=16]{
			\centering
			\includegraphics[scale=0.3]{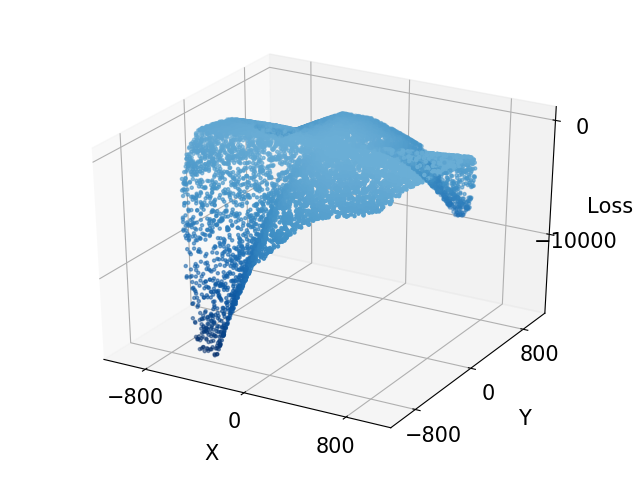}
			\label{fig:l-100-10-16}}
        \hspace{0.01in}
        \subfigure[Input scale=100, hidden size=128]{
		\centering
		\includegraphics[scale=0.3]{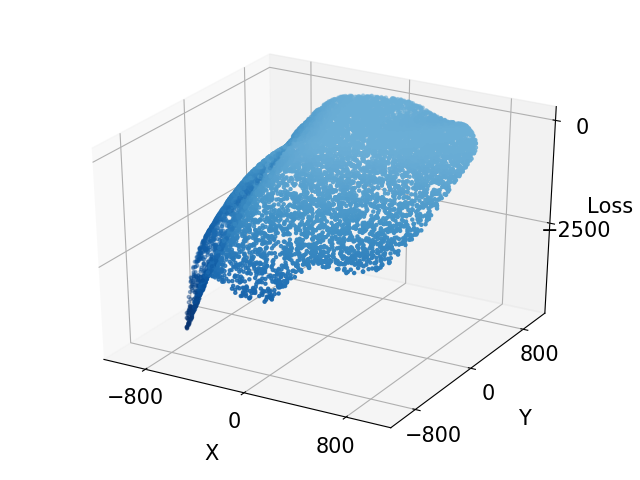}
		\label{fig:l-100-10-128}}
        \hspace{0.01in}
        \subfigure[Input scale=100, hidden size=1024]{
		\centering
		\includegraphics[scale=0.3]{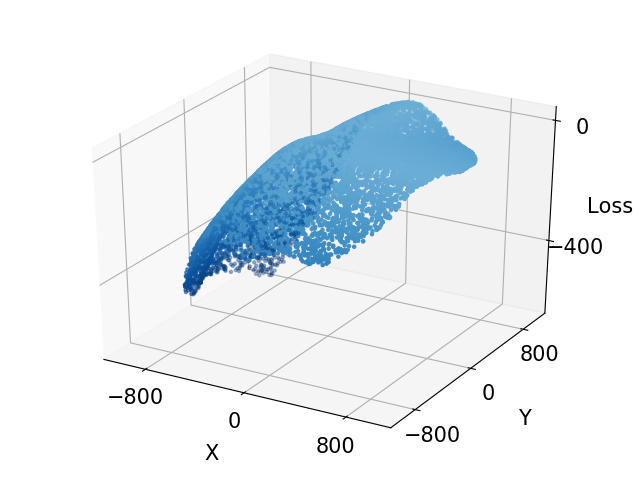}
		\label{fig:l-100-10-1024}}
        \hspace{0.01in}
        
        \caption{Landscapes of adversarial losses on simulated data with different hidden sizes and different input scales. We here fix the perturbation ratio as $10$. We find that wider neural networks lead to more regular landscapes in general. }
		\label{fig:landscapes-hidden-size}
\end{figure*}

\begin{figure*}[t]
		\centering
		
		\subfigure[Input scale=0.01, ratio=0.1]{
			\centering
			\includegraphics[scale=0.3]{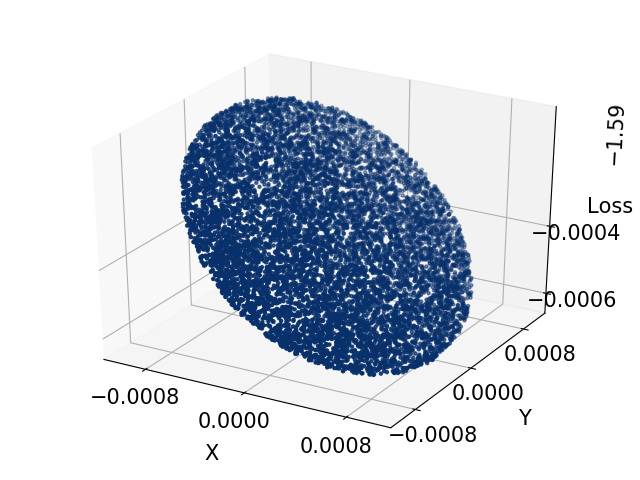}
			\label{fig:l-0.01-0.1-6666}}
        \hspace{0.01in}
        \subfigure[Input scale=0.01, ratio=1.0]{
		\centering
		\includegraphics[scale=0.3]{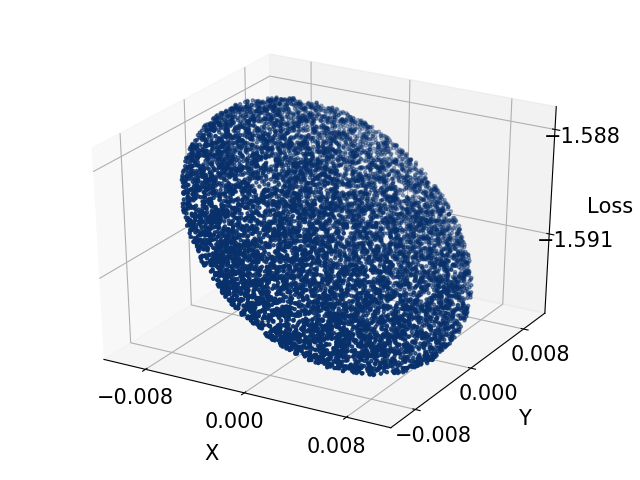}
		\label{fig:l-0.01-1.0-6666}}
        \hspace{0.01in}
        \subfigure[Input scale=0.01, ratio=10]{
		\centering
		\includegraphics[scale=0.3]{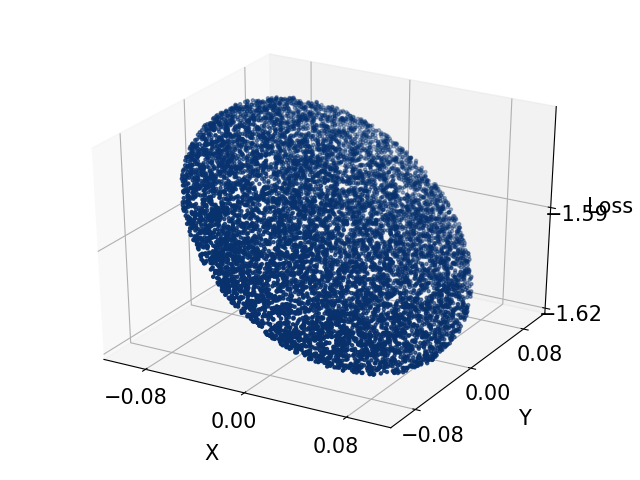}
		\label{fig:l-0.01-10-6666}}
        \hspace{0.01in}
        
        \subfigure[Input scale=1.0, ratio=0.1]{
			\centering
			\includegraphics[scale=0.3]{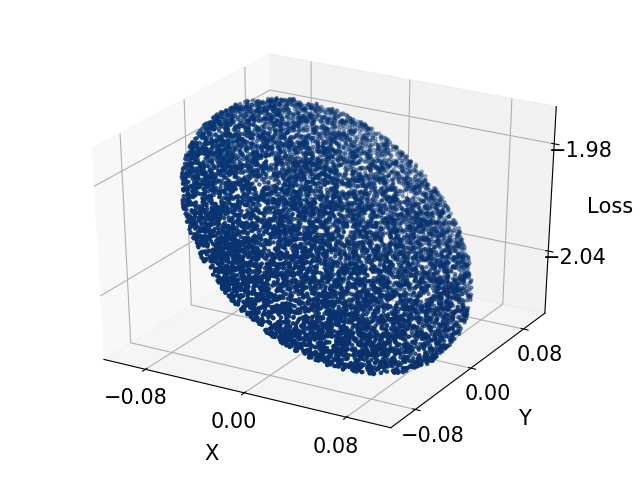}
			\label{fig:l-1.0-0.1-6666}}
        \hspace{0.01in}
        \subfigure[Input scale=1.0, ratio=1.0]{
		\centering
		\includegraphics[scale=0.3]{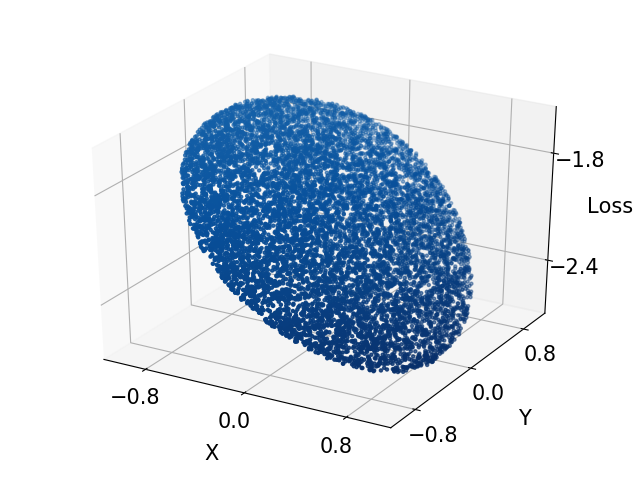}
		\label{fig:l-1.0-1.0-6666}}
        \hspace{0.01in}
        \subfigure[Input scale=1.0, ratio=10]{
		\centering
		\includegraphics[scale=0.3]{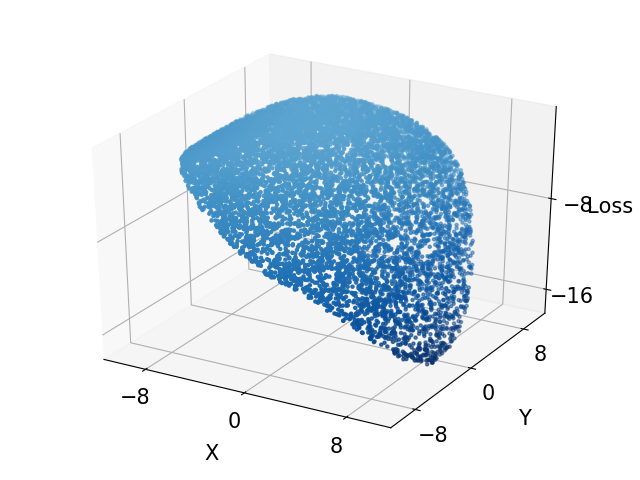}
		\label{fig:l-1.0-10-6666}}
        \hspace{0.01in}
     
        \subfigure[Input scale=100, ratio=0.1]{
			\centering
			\includegraphics[scale=0.3]{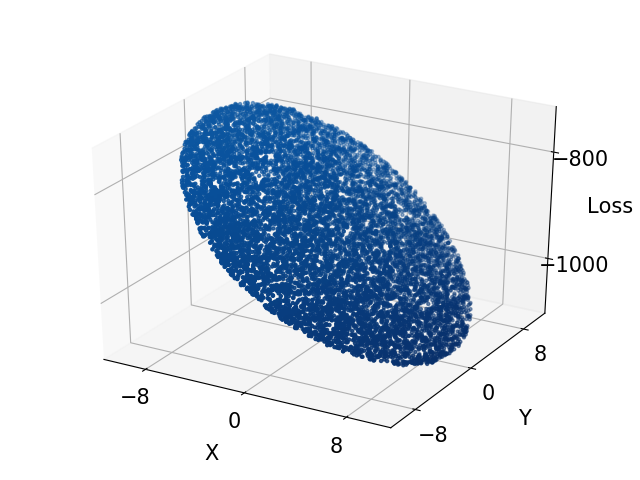}
			\label{fig:l-100-0.1-6666}}
        \hspace{0.01in}
        \subfigure[Input scale=100, ratio=1.0]{
		\centering
		\includegraphics[scale=0.3]{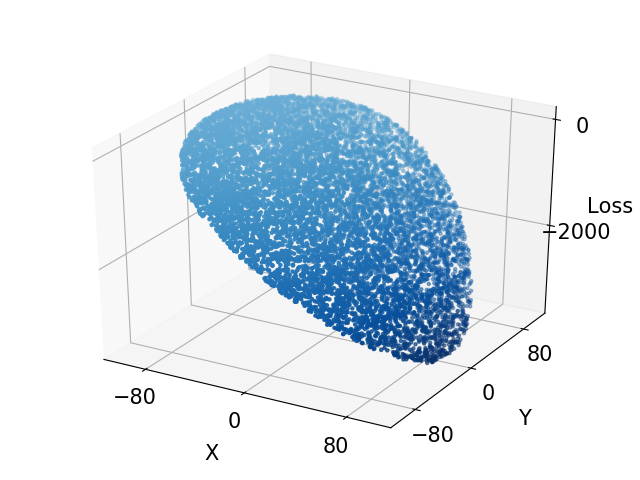}
		\label{fig:l-100-1.0-6666}}
        \hspace{0.01in}
        \subfigure[Input scale=100, ratio=10]{
		\centering
		\includegraphics[scale=0.3]{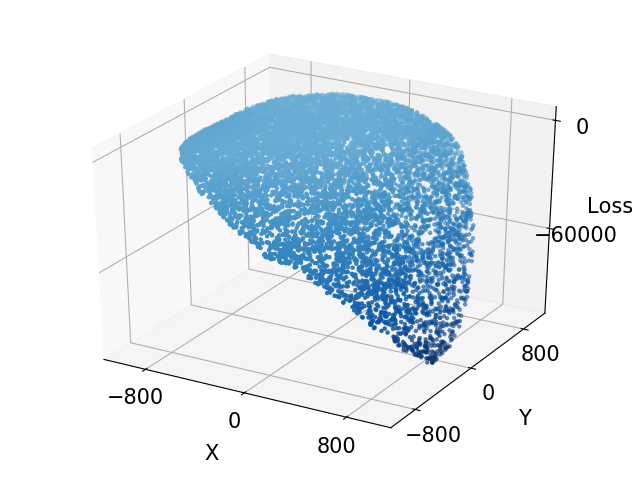}
		\label{fig:l-100-10-6666}}
        \hspace{0.01in}
        
        \caption{Landscapes of adversarial losses on simulated data with another random seed. }
		\label{fig:landscapes-6666}
\end{figure*}

We further analyze the impact of the hidden layer's width on the landscapes in Fig \ref{fig:landscapes-hidden-size} and the landscapes with a different random seed  are shown in Fig \ref{fig:landscapes-6666}. We try several different random seeds and find that the results are all consistent with our analysis. More details can be found in the code.


%% file: main_advpaper.bbl
\begin{thebibliography}{10}

\bibitem{agarwal2018learning}
Naman Agarwal, Alon Gonen, and Elad Hazan.
\newblock Learning in non-convex games with an optimization oracle.
\newblock {\em arXiv preprint arXiv:1810.07362}, 2018.

\bibitem{athalye2018obfuscated}
Anish Athalye, Nicholas Carlini, and David Wagner.
\newblock Obfuscated gradients give a false sense of security: Circumventing
  defenses to adversarial examples.
\newblock {\em arXiv preprint arXiv:1802.00420}, 2018.

\bibitem{brendel2017decision}
Wieland Brendel, Jonas Rauber, and Matthias Bethge.
\newblock Decision-based adversarial attacks: Reliable attacks against
  black-box machine learning models.
\newblock {\em arXiv preprint arXiv:1712.04248}, 2017.

\bibitem{carlini2017towards}
Nicholas Carlini and David Wagner.
\newblock Towards evaluating the robustness of neural networks.
\newblock In {\em 2017 ieee symposium on security and privacy (sp)}, pages
  39--57. IEEE, 2017.

\bibitem{chen2017zoo}
Pin-Yu Chen, Huan Zhang, Yash Sharma, Jinfeng Yi, and Cho-Jui Hsieh.
\newblock Zoo: Zeroth order optimization based black-box attacks to deep neural
  networks without training substitute models.
\newblock In {\em Proceedings of the 10th ACM Workshop on Artificial
  Intelligence and Security}, pages 15--26, 2017.

\bibitem{devlin2018bert}
Jacob Devlin, Ming-Wei Chang, Kenton Lee, and Kristina Toutanova.
\newblock Bert: Pre-training of deep bidirectional transformers for language
  understanding.
\newblock {\em arXiv preprint arXiv:1810.04805}, 2018.

\bibitem{du2017gradient}
Simon~S Du, Chi Jin, Jason~D Lee, Michael~I Jordan, Aarti Singh, and Barnabas
  Poczos.
\newblock Gradient descent can take exponential time to escape saddle points.
\newblock In {\em Advances in neural information processing systems}, pages
  1067--1077, 2017.

\bibitem{du2018gradient}
Simon~S Du, Xiyu Zhai, Barnabas Poczos, and Aarti Singh.
\newblock Gradient descent provably optimizes over-parameterized neural
  networks.
\newblock {\em arXiv preprint arXiv:1810.02054}, 2018.

\bibitem{gao2019convergence}
Ruiqi Gao, Tianle Cai, Haochuan Li, Liwei Wang, Cho-Jui Hsieh, and Jason~D Lee.
\newblock Convergence of adversarial training in overparametrized networks.
\newblock {\em arXiv preprint arXiv:1906.07916}, 2019.

\bibitem{ge2015escaping}
Rong Ge, Furong Huang, Chi Jin, and Yang Yuan.
\newblock Escaping from saddle points—online stochastic gradient for tensor
  decomposition.
\newblock In {\em Conference on Learning Theory}, pages 797--842, 2015.

\bibitem{glorot2010understanding}
Xavier Glorot and Yoshua Bengio.
\newblock Understanding the difficulty of training deep feedforward neural
  networks.
\newblock In {\em Proceedings of the thirteenth international conference on
  artificial intelligence and statistics}, pages 249--256, 2010.

\bibitem{goodfellow2014explaining}
Ian~J Goodfellow, Jonathon Shlens, and Christian Szegedy.
\newblock Explaining and harnessing adversarial examples.
\newblock {\em arXiv preprint arXiv:1412.6572}, 2014.

\bibitem{guo2017countering}
Chuan Guo, Mayank Rana, Moustapha Cisse, and Laurens Van Der~Maaten.
\newblock Countering adversarial images using input transformations.
\newblock {\em arXiv preprint arXiv:1711.00117}, 2017.

\bibitem{he2016deep}
Kaiming He, Xiangyu Zhang, Shaoqing Ren, and Jian Sun.
\newblock Deep residual learning for image recognition.
\newblock In {\em Proceedings of the IEEE conference on computer vision and
  pattern recognition}, pages 770--778, 2016.

\bibitem{ilyas2019adversarial}
Andrew Ilyas, Shibani Santurkar, Dimitris Tsipras, Logan Engstrom, Brandon
  Tran, and Aleksander Madry.
\newblock Adversarial examples are not bugs, they are features.
\newblock In {\em Advances in Neural Information Processing Systems}, pages
  125--136, 2019.

\bibitem{jacot2018neural}
Arthur Jacot, Franck Gabriel, and Cl{\'e}ment Hongler.
\newblock Neural tangent kernel: Convergence and generalization in neural
  networks.
\newblock In {\em Advances in neural information processing systems}, pages
  8571--8580, 2018.

\bibitem{jin2017escape}
Chi Jin, Rong Ge, Praneeth Netrapalli, Sham~M Kakade, and Michael~I Jordan.
\newblock How to escape saddle points efficiently.
\newblock In {\em Proceedings of the 34th International Conference on Machine
  Learning-Volume 70}, pages 1724--1732. JMLR. org, 2017.

\bibitem{liu2019rob}
Xuanqing Liu and Cho-Jui Hsieh.
\newblock Rob-gan: Generator, discriminator, and adversarial attacker.
\newblock In {\em Proceedings of the IEEE Conference on Computer Vision and
  Pattern Recognition}, pages 11234--11243, 2019.

\bibitem{madry2017towards}
Aleksander Madry, Aleksandar Makelov, Ludwig Schmidt, Dimitris Tsipras, and
  Adrian Vladu.
\newblock Towards deep learning models resistant to adversarial attacks.
\newblock {\em arXiv preprint arXiv:1706.06083}, 2017.

\bibitem{papernot2017practical}
Nicolas Papernot, Patrick McDaniel, Ian Goodfellow, Somesh Jha, Z~Berkay Celik,
  and Ananthram Swami.
\newblock Practical black-box attacks against machine learning.
\newblock In {\em Proceedings of the 2017 ACM on Asia conference on computer
  and communications security}, pages 506--519, 2017.

\bibitem{shafahi2019adversarial}
Ali Shafahi, Mahyar Najibi, Mohammad~Amin Ghiasi, Zheng Xu, John Dickerson,
  Christoph Studer, Larry~S Davis, Gavin Taylor, and Tom Goldstein.
\newblock Adversarial training for free!
\newblock In {\em Advances in Neural Information Processing Systems}, pages
  3353--3364, 2019.

\bibitem{szegedy2013intriguing}
Christian Szegedy, Wojciech Zaremba, Ilya Sutskever, Joan Bruna, Dumitru Erhan,
  Ian Goodfellow, and Rob Fergus.
\newblock Intriguing properties of neural networks.
\newblock {\em arXiv preprint arXiv:1312.6199}, 2013.

\bibitem{wang2019convergence}
Yisen Wang, Xingjun Ma, James Bailey, Jinfeng Yi, Bowen Zhou, and Quanquan Gu.
\newblock On the convergence and robustness of adversarial training.
\newblock In {\em International Conference on Machine Learning}, pages
  6586--6595, 2019.

\bibitem{weng2018towards}
Tsui-Wei Weng, Huan Zhang, Hongge Chen, Zhao Song, Cho-Jui Hsieh, Duane Boning,
  Inderjit~S Dhillon, and Luca Daniel.
\newblock Towards fast computation of certified robustness for relu networks.
\newblock {\em arXiv preprint arXiv:1804.09699}, 2018.

\bibitem{wong2018scaling}
Eric Wong, Frank Schmidt, Jan~Hendrik Metzen, and J~Zico Kolter.
\newblock Scaling provable adversarial defenses.
\newblock In {\em Advances in Neural Information Processing Systems}, pages
  8400--8409, 2018.

\bibitem{yin2018rademacher}
Dong Yin, Kannan Ramchandran, and Peter Bartlett.
\newblock Rademacher complexity for adversarially robust generalization.
\newblock {\em arXiv preprint arXiv:1810.11914}, 2018.

\bibitem{zhang2018efficient}
Huan Zhang, Tsui-Wei Weng, Pin-Yu Chen, Cho-Jui Hsieh, and Luca Daniel.
\newblock Efficient neural network robustness certification with general
  activation functions.
\newblock In {\em Advances in neural information processing systems}, pages
  4939--4948, 2018.

\end{thebibliography}
